 \documentclass[letterpaper, 10 pt, conference]{ieeeconf}  % Comment this line out if you need a4paper

\IEEEoverridecommandlockouts                              % This command is only needed if 
\usepackage{graphicx}
\usepackage{amsmath}
\usepackage{amssymb}
\usepackage{subfigure}
\usepackage{color}
\usepackage{cite}

\newtheorem{definition}{Definition}
\newtheorem{lemma}{Lemma}
\newtheorem{theorem}{Theorem}

\newcommand{\norm}[1]{\left\Vert#1\right\Vert} % Norm
\newcommand{\utimes}{ {\raisebox{-0.6ex}{ \kern-1.0ex\raisebox{0.6ex}{\small$\mathsf{v}$}}} } 

\newcommand{\bbm}{\begin{bmatrix}}
\newcommand{\ebm}{\end{bmatrix}}
\newcommand{\bma}[1]{\left[\begin{array}{#1}}
\newcommand{\ema}{\end{array}\right]}

\newcommand{\trans}{{\ensuremath{\mathsf{T}}}} % transpose
\newcommand{\trace}{ {\ensuremath{\mathrm{tr}}} } % \trace
\newcommand{\bnu}{\mbs{\nu}}
\newcommand{\bone}{\mbf{1}}

\newcommand{\onehalf}{\mbox{$\textstyle{\frac{1}{2}}$}}

\newcommand{\beq}{\begin{equation}}
\newcommand{\eeq}{\end{equation}}
\newcommand{\bdis}{\begin{displaymath}}
\newcommand{\edis}{\end{displaymath}}
\newcommand{\beqarray}{\begin{eqnarray}}
\newcommand{\eeqarray}{\end{eqnarray}}
\newcommand{\beqarraynn}{\begin{eqnarray*}}
\newcommand{\eeqarraynn}{\end{eqnarray*}}

\DeclareMathAlphabet{\mbf}{OT1}{ptm}{b}{n}
\newcommand{\mbs}[1]{{\boldsymbol{#1}}}
\newcommand{\mbc}[1]{ \boldsymbol{\mathcal{#1}} }

\newcommand{\mbstilde}[1]{{\tilde{\boldsymbol{#1}}}}

\newcommand{\mbfhat}[1]{{\hat{\mbf{#1}}}}

\newcommand{\mbftilde}[1]{{\tilde{\mbf{#1}}}}

\newcommand{\rev}[1]{{\color{black}{#1}}}

\newtheorem{corollary}{Corollary}

\title{\LARGE \bf
Adaptive Passivity-Based Pose Tracking Control of Cable-Driven Parallel Robots for Multiple Attitude Parameterizations
}

\author{Sze Kwan~Cheah,~\IEEEmembership{Student~Member,~IEEE,}
        Alex~Hayes,~\IEEEmembership{Student~Member,~IEEE,} \\
        and~Ryan~J.~Caverly,~\IEEEmembership{Member,~IEEE}% <-this % stops a space
\thanks{SK. Cheah, A. Hayes, and R. J. Caverly are with the Department
of Aerospace Engineering and Mechanics, University of Minnesota, Minneapolis,
MN, 55455 USA e-mails: \texttt{cheah013@umn.edu}, \texttt{hayes455@umn.edu}, \texttt{rcaverly@umn.edu}.}}

\begin{document}
\maketitle
\thispagestyle{empty}
\pagestyle{empty}

%%%%%%%%%%%%%%%%%%%%%%%%%%%%%%%%%%%%%%%%%%%%%%%%%%%%%%%%%%%%%%%%%%%%%%%%%%%%%%%%
\begin{abstract}

This paper presents a pose tracking controller for a six degree-of-freedom over-constrained cable-driven robot (CDPR). 
The proposed control method uses an adaptive feedforward-based controller to establish a passive input-output mapping for the CDPR that is used alongside a \rev{linear time-invariant strictly positive real} feedback controller to guarantee robust closed-loop input-output stability and asymptotic pose trajectory tracking via the passivity theorem. A novelty of the proposed controller is its formulation for use with a range of payload attitude parameterizations, including any unconstrained attitude parameterization, the quaternion, or the direction cosine matrix (DCM). %as well as pose regulation control law for multiple attitude parameterization that assures closed loop asymptotic stability. 
The performance and robustness of the proposed controller is demonstrated through numerical simulations of a CDPR with rigid and flexible cables. %where the DCM-based controller shows superior attitude tracking compared to the other attitude parameterizations, including an Euler-angle sequence.
\rev{The results demonstrate the importance of carefully defining the CDPR's pose error, which is performed in multiplicative fashion when using the quaternion and DCM, and in a specific additive fashion when using unconstrained attitude parameters (e.g., an Euler-angle sequence).}
%A numerical CDPR example with flexible cables is presented to illustrate the implementation of the proposed control method.

\end{abstract}

% \begin{IEEEkeywords}
% Cable-driven parallel robots, motion control, robust control, adaptive control, attitude parameterizations.
% \end{IEEEkeywords}

%%%%%%%%%%%%%%%%%%%%%%%%%%%%%%%%%%%%%%%%%%%%%%%%%%%%%%%%%%%%%%%%%%%%%%%%%%%%%%%%
\section{INTRODUCTION}

\label{sec:Intro}

Over-constrained cable-driven parallel robots (CDPRs) are a class of parallel robots that make use of a redundant set of tensile cable forces to actuate an end-effector or payload. 
CDPRs typically feature large workspaces and are capable of relatively high payload accelerations due to their low inertia compared to traditional parallel and serial robotic manipulators. Accurate and robust pose (position and attitude/orientation) control of the CDPR's payload or end-effector is challenging, as over-constrained CDPRs are redundantly actuated, which requires a force distribution algorithm (see~\cite{Pott2018} for a summary of commonly used methods), and they can have highly uncertain dynamics (e.g., payload with uncertain inertia or flexible/sagging cables). \rev{Uncertainty in the CDPR's dynamics can be accounted for with adaptive control techniques, which are often coupled to a specific form of a feedback controller (e.g., a constant-gain proportional-derivative controller) and a specific representation of the payload's attitude (e.g., an Euler-angle sequence)~\cite{Lamaury2013,babaghasabha2016adaptive,ji2020adaptive,shang2020adaptive,harandi2021adaptive}.}% \cite{shang2020adaptive}\rjc{Need to add references}\csk{Harandi \cite{harandi2021adaptive} and Babaghasabha \cite{babaghasabha2016adaptive} have CDPR and adaptive in their title that are both journals.}

%To control the pose of the CDPR's payload, positive cable tensions are necessary to avoid cable slack while accounting for redundancy in the system.

%CDPR pose control laws are often calculated in task-space (i.e., in terms of the payload states), resulting in a task-space control wrench~\cite{Bruckmann2008,Zarebidoki2011,Lamaury2013,Schenk2018,Begey2019} that is then transformed into individual positive cable forces or winch torques using force distribution. A number of force distribution methods are found in the literature, many of which are summarized and compared in~\cite{Pott2014} and~\cite[pp.~85--108]{Pott2018}.

%Feedback linearization method to control rigid cable CDPR was demonstrated in~\cite{KorayemFeedbackLinearizationCDPR}.  

Passivity-based control is capable of providing guarantees of robust closed-loop input-output stability for large ranges of system uncertainty and has been widely implemented on serial robotic manipulators \rev{for trajectory tracking}~\cite{Ortega1989,Brogliato2020}. Passivity-based control has recently been extended to \rev{the robust control of} parallel robots~\cite{shibata2008null,abdellatif2008passivity,Hayes2020}, and in particular, CDPRs~\cite{Zarebidoki2011,Caverly2015_TCST,Caverly2018,Godbole2019,khalilpour2021tip,Cheah2021}. %This control approach is especially applicable to CDPRs with unknown payload or flexible cables, where accurate knowledge of the CDPR's dynamics is very difficult to obtain. 
For example, a robust adaptive passivity-based control method for \rev{single degree-of-freedom (DOF)} CDPRs capable of tracking desired payload trajectories in the presence of model uncertainty and flexible cables was presented in~\cite{Godbole2019}. Early work on the passivity-based control of CDPRs focused on suspended CDPRs with the same number of cables as payload DOFs~\cite{Zarebidoki2011} or relied on having twice as many cables as payload DOFs in the overconstrained case~\cite{Caverly2018,Godbole2019,Caverly2015_TCST}. The work of~\cite{Hayes2020,khalilpour2021tip,Cheah2021} demonstrated that passivity-based task-space translational~\cite{Hayes2020,khalilpour2021tip} and pose~\cite{Cheah2021} control of CDPRs can be decoupled from the choice of control allocation method, which greatly expanded its applicability to realistic CDPR configurations that typically feature one or two more cables than payload DOFs. 

Virtually all CDPR pose regulation and tracking controllers in the literature make use of Euler angles to compute the attitude portion of the control law (see examples in~\cite{Lamaury2013,babaghasabha2016adaptive,KorayemFeedbackLinearizationCDPR,ji2020adaptive,shang2020adaptive,harandi2021adaptive,Bruckmann2008,Zarebidoki2011,chellal2017model,Schenk2018,Begey2019,santos2020redundancy}), with the exception of the direction cosine matrix (DCM)-based controller in~\cite{Cheah2021} and the rotation vector-based controller in~\cite{Zake2019}. In other words, the attitude of the CDPR payload at a given instance in time is computed in terms of an Euler-angle sequence and subtracted from a \rev{set of desired Euler angles} to form an error signal that is regulated to zero. Although these Euler-angle-based controllers clearly work in practice, it is unnecessary to restrict CDPR pose control to this one choice of attitude parameterization, especially when the rotation matrix or DCM describing the attitude of the CDPR payload is typically available though the forward kinematics needed to operate the CDPR. In addition, advances in nonlinear pose estimation has led to methods that directly estimate the rotation matrix/DCM~\cite{LeNguyenVinh2021CPRP} or quaternion~\cite{zake2021moving} associated with a CDPR payload, making these quantities readily available for control. %Moreover, Euler-angle-based attitude controllers are typically less responsive and difficult to tune in the presence of larger attitude errors compared to alternative attitude parameterizations, such as the quaternion, rotation vector, and DCM \rjc{We know this to be true (and even show this in our results), but it would be great to have a paper we can cite here.}. 

This paper presents an adaptive passivity-based CDPR pose tracking controller that uses the passivity theorem to guarantee closed-loop input-output stability and asymptotic tracking of a desired payload pose trajectory, where various attitude parameterizations of the payload attiitude can be used.  The proposed controller takes inspiration from multiple sources, including passivity-based adaptive controllers designed for CDPRs~\cite{harandi2021adaptive}, redundantly-actuated flexible manipulators\cite{Damaren1996}, and spacecraft~\cite{egeland1994passivity}. The novel contributions of the proposed controller compared to other adaptive CDPR controllers in the literature, including~\cite{Lamaury2013,babaghasabha2016adaptive,ji2020adaptive,shang2020adaptive,harandi2021adaptive}, is 1) its ability to make use of any unconstrained attitude parameterization, the quaternion, or the DCM when computing the pose tracking error and 2) its ease of use with any input-strictly passive (ISP) or strictly positive real (SPR) feedback controller. The first contribution has the potential to lead to a more homogeneous CDPR operation framework, where the same attitude parameterization can be used both for kinematics and motion control. At a minimum, the proposed control method provides the CDPR operator with a choice as to which attitude parameterization they desire to use for feedback, which, to the best of the knowledge of the authors, is a limitation in the CDPR literature, where Euler-angle sequences are almost exclusively used for control (exceptions include a DCM-based pose-regulation controller was used in~\cite{Cheah2021} and a rotation-vector-based controller was implemented in~\cite{Zake2019}). The second contribution related to the use of an ISP or SPR controller has practical benefits, as the design of the feedback controller can be decoupled from the closed-loop stability analysis and practical control designs, such as a low-pass control gain can be implemented. 

The form of the proposed passivity-based adaptive controller stems from~\cite{slotine1987adaptive,Ortega1989} and makes use of advances in~\cite{Damaren1996,egeland1994passivity}, where attitude parameterizations were incorporated within passivity-based control. %, where different attitude parameterizations are handled within our method is more related to the work of. %inspired by the flexible manipulator pose tracking controller in~\cite{Damaren1996} designed for any unconstrained attitude parameters and the quaternion-based spacecraft attitude tracking controller in~\cite{egeland1994passivity}. 
The novelty of the proposed controller compared to the theory developed in~\cite{Damaren1996} for flexible manipulators, includes extending its use to the quaternion or DCM, as well as its application and validation on a CDPR. 
The quaternion-based spacecraft attitude controller in~\cite{egeland1994passivity} is extended to CDPR pose tracking to yield the proposed quaternion-based method. % of a CDPR's payload. %The proposed controller makes use of the passivity theorem to guarantee closed-loop input-output stability and is capable of asymptotic tracking of a desired payload pose trajectory. 
The work in this paper is also an extension of the preliminary study on passivity-based pose regulation of a CDPR in~\cite{Cheah2021}, which assumed knowledge of the CDPR dynamics, did not provide any mathematical guarantees of pose tracking error convergence, and was limited to the use of the DCM to represent the attitude of the CDPR's payload. The control method proposed in this paper removes these restrictions and assumptions.

%\section{Preliminaries}
The remainder of this paper proceeds as follows. Important preliminaries, including notation, theorems, and a description of the CDPR kinematics and dynamics are presented in Section~\ref{sec:Prelims}. Section~\ref{sec:ControlFormulation} presents the proposed adaptive passivity-based control formulation using unconstrained attitude parameterizations, the quaternion, and the DCM. Numerical simulation results are included in Section~\ref{sec:Examples}, followed by concluding remarks in Section~\ref{sec:Conclusion}.

\section{Preliminaries}
\label{sec:Prelims}

Notation and theorems used throughout the paper are presented in this section, followed by an overview of the CDPR kinematics and dynamics considered in this work.

\subsection{Notation, Definitions and Theorems}
For this paper, the identity matrix and a matrix of zeros are respectively written as $\mbf{1}$ and $\mbf{0}$. 
Matrices are represented in bold (e.g., $\mbf{A} \in \mathbb{R}^{n \times m}$). 
%\rev{$\mbf{1}_j$ represents a column vector of $0$ with $1$ in it's $i^{th}$ entry.}  % needed for feedforward
Positive definite matrices are represented by $\mbf{A}=\mbf{A}^\trans>0$. 
%\rev{The cross operator, $(\cdot)^\times: \mathbb{R}^3 \to \mathbb{R}^{3 \times 3}$, satisfies $\mbf{a}^\times \mbf{b} = - \mbf{b}^\times \mbf{a}$ for all $\mbf{a}$,~$\mbf{b} \in \mathbb{R}^3$. }
The cross operator, $(\cdot)^\times: \mathbb{R}^3 \to \mathfrak{so}(3) $, is defined as
\bdis
\label{eq:crossproduct}
\mbf{v}^\times = \bbm 
    0 & -v_3 & v_2 \\
    v_3 & 0 & -v_1 \\
    -v_2 & v_1 & 0
\ebm,
\edis
where $\mbf{v}^\trans = \bbm v_1 & v_2 & v_3 \ebm$ and $\mathfrak{so}(3) = \{\mbf{S} \in \mathbb{R}^{3 \times 3} | \mbf{S} + \mbf{S}^\trans = \mbf{0}\}$. The reverse or uncross operator, $(\cdot)^\utimes: \mathbb{R}^{3 \times 3} \to \mathbb{R}^3$, is defined as $\mbf{A}^\utimes = \bbm a_1 & a_2 & a_3 \ebm^\trans$, where $\mbf{A} = -\mbf{A}^\trans = \left(\mbf{A}^\utimes \right)^\times$. 
The antisymmetric projection operator, $\mbc{P}(\cdot): \mathbb{R}^{3 \times 3} \to \mathfrak{so}(3)$, projects a matrix $\mbf{U} \in \mathbb{R}^{3 \times 3}$ to the set of antisymmetric matrices, where $\mbc{P} \left( \mbf{U} \right) = \frac{1}{2} \left(\mbf{U}-\mbf{U}^\trans \right)$.
For $\mbf{v} \in \mathbb{R}^3$ and $\mbf{U} \in \mathbb{R}^{3 \times 3}$, it follows that~\cite{Forbes2013}
\beq
\label{eq:ProjProperty}
\onehalf \trace \left(\mbf{v}^\times \mbf{U}\right) = - \mbf{v}^\trans \mbc{P}\left(\mbf{U}\right)^\utimes.
\eeq
Another useful cross operator identity is given by~\cite{WuTse-Huai2016Avof}
\beq
\label{eq:crossidentity}
\mbf{v}^\times \mbf{A} + \mbf{A}^\trans \mbf{v}^\times = ((\trace({\mbf{A}}) \bone - \mbf{A})\mbf{v})^\times,
\eeq
where $\mbf{v} \in \mathbb{R}^3$ and $\mbf{A} \in \mathbb{R}^{3 \times 3}$. 
The signal $\mbf{y}(t)$ satisfies $\mbf{y} \in \mathcal{L}_2$ if $\norm{\mbf{y}}_2^2 = \int_0^\infty \mbf{y}^\trans(t) \mbf{y}(t) \mathrm{d}t < \infty$. The signal $\mbf{y}(t)$ satisfies $\mbf{y} \in \mathcal{L}_{2e}$ if $\mbf{y}_T \in \mathcal{L}_2$ for all $T \in \mathbb{R}_{\geq 0}$, where
$\mbf{y}_T(t) = \mbf{y}(t)$ for $0 \leq t \leq T$ and $\mbf{y}_T(t) = \mbf{0}$ for $T < t$.

% \begin{theorem}[Passivity Theorem~\cite{Desoer}]
% \label{theorem:Passivity}
% Consider the negative feedback interconnection of $\mbc{G}_1: \mathcal{L}_{2e} \rightarrow \mathcal{L}_{2e}$ and $\mbc{G}_2: \mathcal{L}_{2e} \rightarrow \mathcal{L}_{2e}$, defined as $\mbf{y_1} = \mbc{G}_1 \mbf{u}_1$, $\mbf{y}_2 = \mbc{G}_2 \mbf{u}_2$, $\mbf{u}_1 = \mbf{r}_1 - \mbf{y}_2$, and $\mbf{u}_2 = \mbf{y}_1$ where $\mbf{r}_1$ is an exogenous input.  If $\mbc{G}_1$ is passive and $\mbc{G}_2$ is ISP, then the interconnection is input-output stable and $\mbf{y}_1 \in \mathcal{L}_2$ provided $\mbf{r}_1 \in \mathcal{L}_2$.
% \end{theorem}
The attitude of reference frame $\mathcal{F}_p$ relative to reference frame $\mathcal{F}_a$ is described by the DCM $\mbf{C}_{pa}$, which is a member of the special orthogonal group $SO(3)$, where $SO(3) = \{ \mbf{C} \in \mathbb{R}^{3 \times 3} \,\, | \,\, \mbf{C}^\trans \mbf{C} = \mbf{1}, \,\, \det(\mbf{C}) = 1\}$. The DCM $\mbf{C}_{pa}$ is related to the rotation matrix, $\mbf{R}$ that rotates frame $\mathcal{F}_a$ to $\mathcal{F}_p$ by $\mbf{C}_{pa} = \mbf{R}^\trans$. Parameterizations of the DCM are represented in this paper as $\mbf{q}^{pa} \in \mathbb{R}^n$, examples of which include an Euler-angle sequence ($\mbf{q}^{pa} \in \mathbb{R}^{3}$), the quaternion ($\mbf{q}^{pa} \in \mathbb{R}^{4}$), or even the columns of the DCM ($\mbf{q}^{pa} \in \mathbb{R}^{9}$). 
Poisson's equation relates the angular velocity to the time derivative of the DCM as $\dot{\mbf{C}}_{pa} = -\mbs{\omega}^{pa^\times} \mbf{C}_{pa}$, where $\mbs{\omega}^{pa}$ is the angular velocity of $\mathcal{F}_p$ relative to $\mathcal{F}_a$ resolved in $\mathcal{F}_p$. The attitude parameterization rates are related to angular velocity by $\mbs{\omega}^{pa} = \mbf{S}(\mbf{q}^{pa}) \dot{\mbf{q}}^{pa}$, where $\mbf{S}(\mbf{q}^{pa})$ is a mapping matrix whose contents depends on the choice of attitude parameterization~\cite{MarkleyF.Landis2014Fosa,nguyenAttitude}. %$\mbf{P}$ encapsulates this as simply 
%\begin{align}
%     \mbf{P} &= \bbm \bone & \bzero \\ \bzero & \mbf{S} \ebm. 
% \end{align}

\begin{definition}[Passivity~\cite{Brogliato2020}]
The input-output mapping $\mbf{u}\mapsto \mbf{y}$ associated with the operator $\mbc{G}: \mathcal{L}_{2e} \rightarrow \mathcal{L}_{2e}$, where $\mbf{y} = \mbc{G}(\mbf{u})$, is ISP if for all $\mbf{u} \in \mathcal{L}_{2e}$ and $T \in \mathbb{R}_{\geq 0}$ there exist $\delta \in \mathbb{R}_{>0}$ and $\beta \in \mathbb{R}$ such that 
\beq
\label{eq:DefISP}
\int_0^T{\mbf{y}^\trans(t)\mbf{u}}(t)\mathrm{d}t \geq \delta \norm{\mbf{u}_T}_2^2 + \beta.
\eeq
If~\eqref{eq:DefISP} is satisfied with $\delta = 0$, then $\mbf{u} \mapsto \mbf{y}$ is passive.  The scalar $\beta$ is a constant related to initial conditions.
\end{definition}

\subsection{CDPR Kinematics and Dynamics}
\begin{figure}[t!]
\centering
\includegraphics[width=0.5\textwidth]{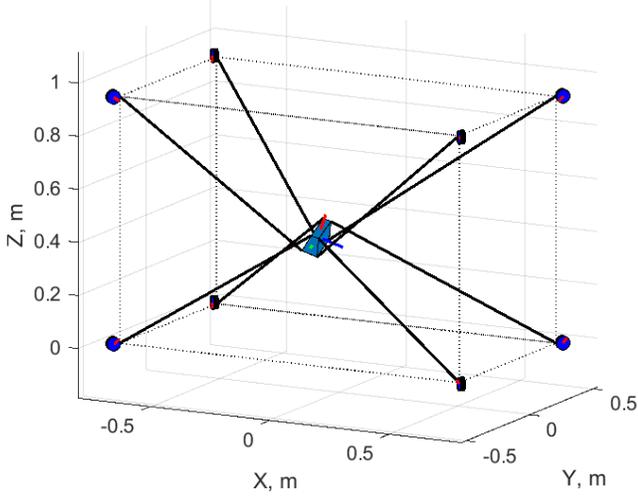}
%\vspace{-10pt}
\caption{A $6$~DOF CDPR with 8 flexible cables and a rigid-body payload.} \label{fig:CDPR_Anim}
%\vspace{-10pt}
\end{figure}
Consider an over-constrained CDPR with $m$ rigid cables actuated by winches and connected to a rigid-body payload, where $m > 6$, as shown in Fig.~\ref{fig:CDPR_Anim}. %It is assumed that none of the winches are colocated. 
The CDPR's equations of motion in task space are given as~\cite{Khosravi2}
\beq
\label{eq:EoMTask}
\mbf{M}(\mbs{\rho}) \dot{\mbs{\nu}} + \mbf{D}(\mbs{\rho},{\mbs{\nu}}) \mbs{\nu} + \mbf{g}(\mbs{\rho}) = \mbs{\Pi}^\trans(\mbs{\rho})\mbs{\tau},
\eeq
where $\mbs{\rho}^\trans = [ \mbf{r}^\trans \,\, \mbf{q}^{pa^\trans} ]$ represents the payload's pose, $\mbf{r} \in \mathbb{R}^3$ is the position of the payload's center of mass relative to a point at the origin of an inertial frame $\mathcal{F}_a$ resolved in $\mathcal{F}_a$, and $\mbf{q}^{pa} \in \mathbb{R}^n$ is an attitude \rev{parameterization} of the payload-fixed reference frame $\mathcal{F}_p$ relative to $\mathcal{F}_a$. The augmented payload velocity is given by $\bnu^\trans = [ \dot{\mbf{r}}^\trans \,\, {\mbs{\omega}}^{pa^\trans} ]$, where $\mbs{\omega}^{pa} \in \mathbb{R}^3$ is the angular velocity of $\mathcal{F}_p$ relative to $\mathcal{F}_a$ resolved in $\mathcal{F}_p$. The torques applied by the winches are denoted as $\mbs{\tau}^\trans = \bbm \tau_1 & \cdots & \tau_m \ebm$. The remaining terms in~\eqref{eq:EoMTask} are the mass matrix $\mbf{M}(\mbs{\rho}) = \mbf{M}^\trans(\mbs{\rho}) > 0$, the nonlinear term $\mbf{D}(\mbs{\rho},\mbs{\nu})$, which contains centrifugal and Coriolis forces, and the gravitational term $\mbf{g}(\mbs{\rho})$. Furthermore, it is known that $\dot{\mbf{M}}(\mbs{\rho}) - 2\mbf{D}(\mbs{\rho},\mbs{\nu})$ is skew-symmetric~\cite{Khosravi2}. The winch torques are distributed \rev{through the wrench matrix $\mbs{\Pi}(\mbs{\rho}) \in \mathbb{R}^{m \times 6}$, which is uniquely defined through inverse velocity kinematics % mapping $\dot{\mbs{\theta}} = \mbs{\Pi}(\mbs{\rho}) \bnu$, where $\mbs{\theta}^\trans = \bbm \theta_1 & \cdots & \theta_m \ebm$ represents the CDPR's $m$ winch angles. The payload 
and is full rank when the payload remains within its wrench-feasible workspace.}%, thus ensuring $\mbs{\Pi}(\mbs{\rho})$ is full rank. %The forward velocity kinematic mapping $\bnu = \mbf{U}(\mbs{\theta}) \dot{\mbs{\theta}}$ is performed with the matrix $\mbf{U}(\mbs{\theta}) \in \mathbb{R}^{6 \times m}$, which satisfies $\mbf{U}(\mbs{\theta}) \mbs{\Pi}(\mbs{\rho}) = \mbf{1}$. %The null space of $\mbs{\Pi}(\mbs{\rho})$ is spanned by the matrix , which represents . 
%This matrix $\mbf{U}(\mbs{\theta})$ is not uniquely defined, which can be exploited when allocating the desired wrench exerted on the end-effector to winch torques.

\section{Control Formulation and Passivity \& Stability Analyses}
\label{sec:ControlFormulation}

The pose tracking controller is presented in this section, \rev{where the objective is to ensure $\mbf{r} \to \mbf{r}_d$ and $\mbf{q}^{pa} \to \mbf{q}^{da}$ as $t \to \infty$, where $\mbf{r}_d$ and $\mbf{q}^{da}$ describe the desired position and attitude trajectories of the CDPR payload, respectively.}
%
%
%
%\subsection{Control Objective and Architecture}
%
\rev{As with previous preliminary work on the passivity-based control of CDPRs~\cite{Cheah2021}, the controller is formulated in terms of % the control wrench % winch torques are chosen to satisfy the control allocation equation
\bdis
\mbf{f} = \mbs{\Pi}^\trans(\mbs{\rho})\mbs{\tau},
\edis
where $\mbf{f} \in \mathbb{R}^6$ is the control wrench to be applied to the payload (i.e., its first three elements are a force resolved in $\mathcal{F}_a$ and its last three elements are a torque resolved in $\mathcal{F}_p$). 
See~\cite{Pott2018} for a summary of a force distribution methods that can be used to determine the torques $\mbs{\tau}$ that generate the desired control wrench. Force distribution is not a contribution of this work and the following analysis and results are valid for any method, provided $\mbf{f} = \mbs{\Pi}^\trans(\mbs{\rho})\mbs{\tau}$.% provided~\eqref{eq:ControlAlloc} is satisfied.
}

%\rjc{Need brief mention of $\mbs{\tau} = \mbf{U}^\trans (\mbs{\theta}) \mbshat{\tau}$. Can point to our ACC paper as example of this and say that control allocation is not a contribution of this work, but method used is described in results section.}
%
\rev{The proposed control input is described by}%. The control input $\hat{\mbs{\tau}}$ is separated into
\begin{align}
    \mbf{f}
    &=
    \mbf{f}_{ff} + \mbf{f}_{fb},
    \label{eq:TorqWithAdapFeedForward}
\end{align}
where $\mbf{f}_{ff}$ is an adaptive feedforward-based input and $\mbf{f}_{fb}$ is a feedback input. 
The remainder of this section outlines the proposed adaptive feedforward-based and feedback control inputs, formulations of the pose tracking errors for different attitude parameterizations, along with proofs of passivity and closed-loop trajectory tracking convergence.

\subsection{Adaptive Feedforward-Based Control}
The adaptive feedforward\rev{-based control} input is derived by first considering a desired feedforward-based input% \rev{of a rigid payload}
\beq
    \mbf{f}_d = 
    \bbm       m_p \mbf{1} & \mbf{0}  \\ \mbf{0} &    \mbf{I}_p     \ebm
    \dot{\mbs{\nu}}_r
    +
    \bbm    \mbf{0} &     \mbf{0} \\ \mbf{0}    &    \mbs{\omega}^{pa^\times}      \mbf{I}_p          \ebm 
    \mbs{\nu}_r
    +
    \bbm        m_pg\mbf{1}_3        \\        \mbf{0}    \ebm
    % \mbf{M}(\mbs{\rho}) \dot{\mbs{\nu}}_r + \mbf{D}(\mbs{\rho},{\mbs{\nu}}) \mbs{\nu}_r + \mbf{g}(\mbs{\rho}),
    \label{eq:feedforward_1}
\eeq
where \rev{$m_p \in \mathbb{R}$ and $\mbf{I}_{p} = \mbf{I}_{p}^\trans \in \mathbb{R}^{3 \times 3}$ represent the constant mass and inertia of the payload, and $\mbf{1}_3^\trans = \bbm 0 & 0 & 1 \ebm$. Note that~\eqref{eq:feedforward_1} is equivalent to~\eqref{eq:EoMTask}, with the assumption that the dynamics of the CDPR are dominated by those of its rigid-body payload and the replacement of $\mbs{\nu}$ by $\mbs{\nu}_r$, which represents} the virtual filtered rate trajectory defined as~\cite{Damaren1996}
\begin{align}
    \bnu_r 
    % &= \bnu_d - \mbs{\Lambda} \mbf{P} \tilde{\mbf{p}},  %corrected per reviewer
    &= \bnu_d - \mbf{P} \mbs{\Lambda} \tilde{\mbf{p}},
    \label{eq:FilteredRates}
\end{align}
where $\mbs{\Lambda} = \mbs{\Lambda}^\trans > 0$ is a control gain, and the terms $\tilde{\mbf{p}} \in \mathbb{R}^6$ and $\bnu_d \in \mathbb{R}^6$ are the pose tracking error and desired augmented velocity, respectively, which are defined in the following subsection for difference choices of attitude parameterizations, along with the matrix $\mbf{P} \in \mathbb{R}^{6 \times 6}$. \rev{The variable $\mbs{\nu}_r$ is related to the virtual reference trajectory in~\cite{slotine1987adaptive}, but is designed in a distinct manner to accommodate various attitude parameterizations, as outlined in Section~\ref{sec:AttParamFormulation}}. Although~\eqref{eq:feedforward_1} includes the desired trajectory, the presence of $\mbftilde{p}$ introduces feedback within $\mbf{f}_d$, which is why the term ``feedforward-based control input'' is used. The feedforward-based input in~\eqref{eq:feedforward_1} can be alternatively written as
\begin{align}
    \label{eq:feedforward}
    \mbf{f}_d &=
    \mbf{W} \mbf{a}, 
\end{align}
where $\mbf{a}^\trans=\bbm m_p & I_{11} & I_{22} & I_{33} & I_{12} & I_{13} & I_{23}  \ebm$, $I_{ij}$ represent the six unique entries of $\mbf{I}_p$, and $\mbf{W} = \frac{\partial \mbf{f}_{d}}{\partial \mbf{a}} $. The term $\mbf{W}$ is a function of $\mbs{\nu}_r$, $\dot{\mbs{\nu}}_r$, and $\mbs{\omega}^{pa}$, while $\mbf{a}$ is equivalent to the minimal parameter formulation developed in~\cite{slotine1987adaptive}.

%is a function of $\dot{\bnu}_r$ and $\bnu_r$ while $\mbf{a}$ includes the variables that make up $\mbf{M}(\mbs{\rho})$, $\mbf{D}(\mbs{\rho})$, and $\mbf{g}(\mbs{\rho})$, expressed in column matrix form.

In practice, the entries of $\mbf{a}$ are not known exactly, so instead an estimate of $\mbf{a}$ is employed, which is denoted as $\hat{\mbf{a}}$. The adaptive control input in~\eqref{eq:TorqWithAdapFeedForward} is defined as
\beq
\label{eq:tau_hat_ff}
\mbf{f}_{ff} = \mbf{W} \hat{\mbf{a}},
\eeq
%
%
% Defining the applied torque as
% \begin{align}
%     \hat{\mbs{\tau}} 
%     &=
%     \mbf{W} \hat{\mbf{a}} + \hat{\mbs{\tau}}_c
%     \label{eq:TorqWithAdapFeedForward}
% \end{align}
where $\hat{\mbf{a}}$ evolves through the adaptive update law
\beq
\label{eq:adapt_update}
\dot{\hat{\mbf{a}}} =
    - \mbs{\Upsilon} \mbf{W}^\trans \tilde{\bnu}_r,
\eeq
and $\mbs{\Upsilon} = \mbs{\Upsilon}^\trans > 0$ is a constant used to adjust the adaptation rate~\cite{Damaren1996}. 
%
%designed in Section~\ref{sec:Analysis} to ensure closed-loop input-output stability and convergence of the trajectory tracking error.
%
% Subtracting~\eqref{eq:feedforward_1} from \eqref{eq:EoMTask} and substituting the expressions for the control inputs in~\eqref{eq:ControlAlloc},~\eqref{eq:TorqWithAdapFeedForward},~\eqref{eq:feedforward}, and~\eqref{eq:tau_hat_ff} yields the error dynamics
Subtracting~\eqref{eq:feedforward_1} from \eqref{eq:EoMTask}, assuming the dynamics of the CDPR in~\eqref{eq:EoMTask} are dominated by those of its rigid-body payload, and substituting the expressions for the control inputs~\eqref{eq:TorqWithAdapFeedForward}, and~\eqref{eq:tau_hat_ff} yields the error dynamics
\beq
\label{eq:errordyn}
 \mbf{M}(\mbs{\rho}) \dot{\tilde{\bnu}}_r
    + \mbf{D}(\mbs{\rho}, \bnu) \tilde{\bnu}_r
    = \mbf{f} - \mbf{f}_d
    = \mbf{W} {\tilde{\mbf{a}}} + \mbf{f}_{fb},%{\tilde{\mbs{\tau}}},
\eeq
where $\mbftilde{a} = \mbfhat{a} - \mbf{a}$ and
\beq
\label{eq:nu_r_tilde}
\tilde{\bnu}_r = \bnu - \bnu_r = \bnu - \left(\bnu_d  -  \mbf{P}\mbs{\Lambda} \tilde{\mbf{p}}\right) = \tilde{\bnu} +  \mbf{P} \mbs{\Lambda}\tilde{\mbf{p}},
\eeq
$\tilde{\bnu} = \bnu - \bnu_d$. %, $\tilde{\mbs{\tau}} = \hat{\mbs{\tau}} - \hat{\mbs{\tau}}_d= \mbf{W} {\tilde{\mbf{a}}} + \hat{\mbs{\tau}}_{fb}$, 
Note that since $\mbf{a}$ is constant, $\dot{\mbftilde{a}} = \dot{\mbfhat{a}}$.

\subsection{Feedback Variable Formulation with Various Parameterizations of Payload Attitude}
\label{sec:AttParamFormulation}

%\subsection{Problem Statement}

One of the main contributions of this paper is extending the control formulation of~\cite{Damaren1996} to accommodate quaternion and $SO(3)$ attitude parameterizations. This extension relies on the derivation of a suitable filtered error system output for each parameterization that yields a passive input-output mapping and fits within the control formulation of~\cite{Damaren1996}.

The filtered error output is of the form
\begin{align}
    \mbf{s} &= \dot{\tilde{\mbf{p}}} 
            + \mbs{\Lambda} \tilde{\mbf{p}}, \label{eq:svar}
\end{align}
where $\mbs{\Lambda} = \mbs{\Lambda}^\trans >0$ is a proportional-like control gain that is also featured in the definition of $\mbs{\nu}_r$ in~\eqref{eq:FilteredRates}. In order to demonstrate a passive input-output mapping, closed-loop input-output stability, and convergence of the pose tracking error in Section~\ref{sec:Analysis}, it is required that $\tilde{\mbs{\nu}}_r = \mbs{\nu} - \mbs{\nu}_r = \mbf{P} \mbf{s}$, where $\mbs{\nu}_r$ is a function of $\mbs{\nu}_d$, $\mbf{P}$, and $\tilde{\mbf{p}}$. The remainder of this subsection focuses on determining suitable choices of $\mbf{P}$, $\tilde{\mbf{p}}$, and $\mbs{\nu}_d$ that ensures this property is satisfied for different choices of attitude parameterizations.

\subsubsection{Unconstrained Attitude Parameterizations}

Unconstrained attitude parameterizations, such as Euler-angle sequences, the rotation vector, and modified Rodrigues parameters (MRPs), are made up of 3 parameters that are free to evolve in time without any constraints, but suffer from singularities at one or more attitudes. For example, the 3-2-1 Euler-angle sequence described by the parameters $\mbf{q}^{pa^\trans} = \bbm \phi & \theta & \psi \ebm$ has a kinematic singularity at $\theta = \pm \pi/2$, which results in the kinematic mapping matrix
\begin{align}
    \mbf{S}(\mbf{q}^{pa}) &= \bbm 
       1 & 0 & -\sin{\theta} \\
       0 & \cos{\phi} & \sin{\phi} \cos{\theta} \\
       0 & -\sin{\phi} & \cos{\phi} \cos{\theta}
    \ebm,
\end{align}
satisfying $\mbs{\omega}^{pa} = \mbf{S}(\mbf{q}^{pa})  \dot{\mbf{q}}^{pa}$, to become singular~\cite{MarkleyF.Landis2014Fosa}.% A summary of the form of the kinematic mapping matrix $\mbf{S}(\mbf{q}^{pa})$ for other choices of unconstrained attitude parameterizations can be found in~\cite{MarkleyF.Landis2014Fosa,nguyenAttitude}. 

% \begin{lemma}
% \label{theorem:unconstrained}
% Consider an unconstrained attitude parameterization $\mbf{q}^{pa} \in \mathbb{R}^3$. The definitions $\mbf{P} = \text{diag}\{\mbf{1},\mbf{S}(\mbf{q}^{pa})\}$, $\mbs{\nu}_d =\mbf{P} \bbm \dot{\mbf{r}}_d \\ \dot{\mbf{q}}^{da}\ebm$, and $\tilde{\mbf{p}} = \bbm \tilde{\mbf{r}} \\ \mbf{q}^{pa} - \mbf{q}^{da} \ebm $ ensure that $\tilde{\mbs{\nu}}_r = \mbf{P} \mbf{s}$.
% \end{lemma}
\begin{lemma}
\label{theorem:unconstrained}
Consider an unconstrained attitude parameterization $\mbf{q}^{pa} \in \mathbb{R}^3$. The definitions
\begin{align}
    \mbf{P} &= \bbm \mbf{1} & \mbf{0} \\ \mbf{0} & \mbf{S}(\mbf{q}^{pa})\ebm, \label{lemma1:P}\\
    \mbs{\nu}_d &=\mbf{P} \bbm \dot{\mbf{r}}_d \\ \dot{\mbf{q}}^{da}\ebm, \label{lemma1:nud}\\
    \tilde{\mbf{p}} &= \bbm \tilde{\mbf{r}} \\ \mbf{q}^{pa} - \mbf{q}^{da} \ebm, \label{lemma1:ptilde}
\end{align}
where $\mbftilde{r} = \mbf{r}-\mbf{r}_d$, ensure that $\tilde{\mbs{\nu}}_r = \mbf{P} \mbf{s}$.
\end{lemma}
\begin{proof}
Substituting~\eqref{lemma1:ptilde} into~\eqref{eq:svar} and multiplying by~\eqref{lemma1:P} results in
% Substituting $\tilde{\mbf{p}} = \bbm \tilde{\mbf{r}} \\ \mbf{q}^{pa} - \mbf{q}^{da} \ebm $ into~\eqref{eq:svar} and multiplying by $\mbf{P} = \text{diag}\{\mbf{1},\mbf{S}(\mbf{q}^{pa})\}$ results in
\beq
\mbf{P}\mbf{s} = \mbf{P} \left(\bbm \dot{\mbf{r}} - \dot{\mbf{r}}_d \\ \dot{\mbf{q}}^{pa} - \dot{\mbf{q}}^{da} \ebm + \mbs{\Lambda} \tilde{\mbf{p}}\right) 
= \mbf{P} \bbm \dot{\mbf{r}} \\ \dot{\mbf{q}}^{pa}\ebm - \mbf{P} \bbm \dot{\mbf{r}}_d \\ \dot{\mbf{q}}^{da}\ebm + \mbf{P}\mbs{\Lambda} \tilde{\mbf{p}}. \label{eq:Lemma1_proof1}
\eeq
Multiplying out the first term in~\eqref{eq:Lemma1_proof1} and using the fact that $\mbs{\omega}^{pa} = \mbf{S}(\mbf{q}^{pa})  \dot{\mbf{q}}^{pa}$ yields
\beq
\label{eq:Lemma1_proof2}
\mbf{P} \bbm \dot{\mbf{r}} \\ \dot{\mbf{q}}^{pa}\ebm = \bbm \dot{\mbf{r}} \\ \mbf{S}(\mbf{q}^{pa}) \dot{\mbf{q}}^{pa} \ebm = \bbm \dot{\mbf{r}} \\ \mbs{\omega}^{pa} \ebm = \bnu.
\eeq
Substituting~\eqref{eq:Lemma1_proof2} into~\eqref{eq:Lemma1_proof1} and using~\eqref{eq:nu_r_tilde} and~\eqref{lemma1:ptilde} gives
$\mbf{P} \mbf{s} = \tilde{\bnu} + \mbf{P}\mbs{\Lambda} \tilde{\mbf{p}} = \tilde{\mbs{\nu}}_r$.
\end{proof}
As in~\cite{Damaren1996}, the definition of $\mbs{\nu}_d$ involves evaluating $\mbf{S}(\mbf{q}^{pa})$ with the payload attitude and not the desired attitude.

\subsubsection{Quaternion}
The quaternion $\mbf{q}^{pa^\trans} = \bbm \mbs{\epsilon}^\trans & \eta \ebm$ is composed of the vector portion $\mbs{\epsilon} \in \mathbb{R}^3$ and scalar part $\eta \in \mathbb{R}$, which satisfy the constraint $\mbf{q}^{pa^\trans} \mbf{q}^{pa} = \mbs{\epsilon}^\trans \mbs{\epsilon} + \eta^2 = 1$.
The quaternion error is defined in~\cite{MarkleyF.Landis2014Fosa,egeland1994passivity} as
\beq 
    \label{eq:quat_err}
    \delta \mbf{q} = \bbm \delta \mbs{\epsilon} \\ \delta \eta \ebm
    = 
    \bbm 
        \eta \bone - \mbs{\epsilon}^\times & \mbs{\epsilon} \\
        -\mbs{\epsilon}^\trans & \eta
    \ebm
    \bbm -\mbs{\epsilon}_d \\ \eta_d \ebm 
    % \frac{1}{\norm{\mbf{q}^{da}}_2^2}
    ,
\eeq
where $\mbf{q}^{da^\trans} = \bbm \mbs{\epsilon}_d^\trans & \eta_d \ebm$ is the desired quaternion. 
%
%This is the quaternion multiplication of the quaternion attitude $\mbf{q}^{pa}$ and the inverse of the desired quaternion attitude $\mbf{q}^{da}$.
% \beq
% \label{eq:quat_err}
% \delta \mbf{q} = \bbm \delta \mbs{\epsilon} \\ \delta \eta \ebm = \mbf{q}^{pa} \otimes \mbf{q}^{da^{-1}},
% \eeq
% where $\mbf{q}_d = \bbm \mbs{\epsilon}_d^\trans & \eta_d \ebm^\trans$ is the desired quaternion attitude, %
% \begin{align}
%     \mbf{q}^{pa} \otimes &= \bbm 
%         \eta \bone - \mbs{\epsilon}^\times & \mbs{\epsilon} \\
%         \mbs{\epsilon}^\trans & \eta
%     \ebm
%     \notag
% \end{align}
% is the quaternion multiplication operator, %$(\cdot)\otimes : \mathbb{R}^4 \to \mathbb{R}^{4 \times 4} $, is defined as
% and $\mbf{q}^{da^{-1}} = \bbm -\mbs{\epsilon}_d^\trans & \eta_d \ebm^\trans / \norm{\mbf{q}^{da}}_2^2$ is the quaternion inverse.
The desired angular velocity is defined as a function of the rate of the desired quaternion as $\mbs{\omega}^{da} = 2
\bbm
    \eta \bone - \mbs{\epsilon}^\times & -\mbs{\epsilon}
\ebm
\dot{\mbf{q}}^{da}$.

\begin{lemma}
\label{theorem:quaternion}
Consider the quaternion attitude parameterization $\mbf{q}^{pa} \in \mathbb{R}^4$. The definitions
\begin{align}
    \mbf{P} &= \bbm \mbf{1} & \mbf{0} \\ \mbf{0} & 2( \delta \eta \bone + \delta \mbs{\epsilon}^\times )^{-1} \ebm, \label{lemma2:P}\\
    \mbs{\nu}_d &= \bbm \dot{\mbf{r}}_d \\ \mbs{\omega}^{da} + 2( \delta \eta \bone + \delta \mbs{\epsilon}^\times )^{-1} \mbs{\omega}^{da^\times} \delta \mbs{\epsilon} \ebm, \label{lemma2:nud} \\
    \tilde{\mbf{p}} &= \bbm \mbftilde{r} \\ \delta \mbs{\epsilon}  \ebm, \label{lemma2:ptilde}
\end{align}
ensure that $\tilde{\mbs{\nu}}_r = \mbf{P} \mbf{s}$, where
%\bdis
%\delta \dot{\mbs{\epsilon}} = -\mbs{\omega}_d^\times \delta \mbs{\epsilon} 
%          + \onehalf ( \delta \eta \bone + \delta \mbs{\epsilon}^\times ) \delta \mbs{\omega}
%\edis
$\delta \mbs{\epsilon}$ and $\delta \eta$ are defined in~\eqref{eq:quat_err}.
%the filtered body frame rate error $\tilde{\bnu}_r$ of Eq.\eqref{eq:nur} has $\mbf{S} = 2( \delta \eta \bone + \delta \mbs{\epsilon}^\times )^{-1}$ while the attitude error is $\delta \mbs{\epsilon}$. 
\end{lemma}
\begin{proof}
Substituting~\eqref{lemma2:ptilde} into~\eqref{eq:svar} and multiplying by~\eqref{lemma2:P} results in
\beq
\mbf{P}\mbf{s} = \mbf{P} \left(\bbm \dot{\tilde{\mbf{r}}}  \\ \delta \dot{\mbs{\epsilon}} \ebm + \mbs{\Lambda} \tilde{\mbf{p} }\right) = \bbm \dot{\tilde{\mbf{r}}}  \\ 2( \delta \eta \bone + \delta \mbs{\epsilon}^\times )^{-1}\delta \dot{\mbs{\epsilon}} \ebm + \mbf{P}\mbs{\Lambda} \tilde{\mbf{p}}. \label{eq:Lemma2_proof1}
\eeq
The term $\delta \dot{\mbs{\epsilon}}$ can be expanded using the property~\cite{MarkleyF.Landis2014Fosa} $$\delta \dot{\mbs{\epsilon}} = -\mbs{\omega}^{da^\times} \delta \mbs{\epsilon} 
          + \onehalf ( \delta \eta \bone + \delta \mbs{\epsilon}^\times ) \left(\mbs{\omega}^{pa} - \mbs{\omega}^{da}\right).$$
% \beq
% \label{eq:Lemma2_proof2}
% \delta \dot{\mbs{\epsilon}} = -\mbs{\omega}^{da^\times} \delta \mbs{\epsilon} 
%           + \onehalf ( \delta \eta \bone + \delta \mbs{\epsilon}^\times ) \left(\mbs{\omega}^{pa} - \mbs{\omega}^{da}\right).
% \eeq
Substituting this into~\eqref{eq:Lemma2_proof1} and making use of~\eqref{lemma2:nud} yields
\begin{align*}
\mbf{P} \mbf{s} &= \bbm \dot{\tilde{\mbf{r}}}  \\ \mbs{\omega}^{pa} - \mbs{\omega}^{da} - 2( \delta \eta \bone + \delta \mbs{\epsilon}^\times )^{-1}\mbs{\omega}^{da^\times} \delta \mbs{\epsilon} \ebm + \mbf{P}\mbs{\Lambda} \tilde{\mbf{p}} \\
&= \tilde{\bnu} + \mbf{P}\mbs{\Lambda} \tilde{\mbf{p}} = \tilde{\mbs{\nu}}_r.
\end{align*}
\end{proof}
The inverse of the matrix $( \delta \eta \bone + \delta \mbs{\epsilon}^\times )$ that is used to define $\mbf{P}$ exists provided $\delta \eta \neq 0$. This singularity is avoided as long as $\mathcal{F}_p$ and $\mathcal{F}_d$ are within a $\pm \pi/2$~rad rotation of each other, which is to be expected for overconstrained CDPRs.% for all rotation vectors, the matrix remains invertible for the entire range of attitudes considered in this work.

\subsubsection{$SO(3)$ (The Direction Cosine Matrix)}
The DCM can be used directly with the antisymmetric projection operator to form an attitude error and satisfy the desired property. 
\begin{lemma}
\label{theorem:so3}

Consider an $SO(3)$ description of attitude with the DCM $\mbf{C}_{pa} \in SO(3)$. The definitions
\begin{align}
    \mbf{P} &= \bbm \mbf{1} & \mbf{0} \\ \mbf{0} & -2 \left(
        (\trace(\mbf{C}_{pd} ) \bone - \mbf{C}_{pd} ) 
    \right)^{-1} \ebm, \label{lemma3:P}\\
    \mbs{\nu}_d &= \bbm \dot{\mbf{r}}_d \\ \mbs{\omega}^{da} \ebm, \label{lemma3:nud}\\ 
    \tilde{\mbf{p}} &= \bbm \tilde{\mbf{r}} \\ \mbc{P}(\mbf{C}_{pd})^\mathsf{V} \ebm, \label{lemma3:ptilde}
\end{align}
ensure that $\tilde{\mbs{\nu}}_r = \mbf{P} \mbf{s}$, where
$\mbf{C}_{pd} = \mbf{C}_{pa} \mbf{C}_{da}^\trans$ and $\mbf{C}_{da}$ represents the desired payload attitude.

% Consider an $SO(3)$ description of attitude with the DCM $\mbf{C}_{pa} \in SO(3)$. The definitions $\mbf{P} = \text{diag}\{\mbf{1},-2 \left(
%         (\trace(\mbf{C}_{pd} ) \bone - \mbf{C}_{pd} ) 
%     \right)^{-1}\}$, $\mbs{\nu}_d = \bbm \dot{\mbf{r}}_d \\ \mbs{\omega}^{da} \ebm$, and $\tilde{\mbf{p}} = \bbm \tilde{\mbf{r}} \\ \mbc{P}(\mbf{C}_{pd})^\mathsf{V} \ebm$ ensure that $\tilde{\mbs{\nu}}_r = \mbf{P} \mbf{s}$, where
% $\mbf{C}_{pd} = \mbf{C}_{pa} \mbf{C}_{da}^\trans$ and $\mbf{C}_{da}$ represents the desired payload attitude.

%
% Consider the SO(3) attit6ude paramterization, the filtered body frame rate error $\tilde{\bnu}_r$ of Eq.\eqref{eq:nur} has $\mbf{S} = -2 \left(
%         (\trace(\mbf{C}_{pd} ) \bone - \mbf{C}_{pd} ) 
%     \right)^{-1} = \mbs{\Gamma}^{-1} $ with $\mbf{q} = \mathcal{P}(\mbf{C}_{pd})^\mathsf{V}$.
\end{lemma}

\begin{proof}
Substituting~\eqref{lemma3:ptilde} into~\eqref{eq:svar} and multiplying by~\eqref{lemma3:ptilde} gives
\beq
\mbf{P}\mbf{s} = \mbf{P} \left(\bbm \dot{\tilde{\mbf{r}}}  \\ \frac{\mathrm{d}}{\mathrm{d}t}\left(\mbc{P}(\mbf{C}_{pd})^\mathsf{V}\right) \ebm + \mbs{\Lambda} \tilde{\mbf{p} }\right). \label{eq:Lemma3_proof1}
\eeq
Poisson's equation, $\dot{\mbf{C}}_{pd} = -\mbs{\omega}^{pd^\times} \mbf{C}_{pd}$, and the identities in~\eqref{eq:ProjProperty} and~\eqref{eq:crossidentity} are used to compute% the time derivative of $\mathcal{P}(\mbf{C}_{pd})^\mathsf{V}$ as
\begin{align}
\frac{\mathrm{d}}{\mathrm{d}t}\left(\mbc{P}(\mbf{C}_{pd})^\mathsf{V}\right) %&= \onehalf \left( \dot{\mbf{C}}_{pd} - \dot{\mbf{C}}_{pd}^\trans \right)^\mathsf{V}
    %\nonumber \\ %%
    &= -\onehalf \left(
        \tilde{\mbs{\omega}}^{\times}{\mbf{C}}_{pd} +  {\mbf{C}}_{pd}^\trans \tilde{\mbs{\omega}}^{\times}
    \right)^\mathsf{V} \nonumber \\
    &= -\onehalf \left(\left(\left(
        \trace(\mbf{C}_{pd} ) \bone - \mbf{C}_{pd} 
    \right)   \tilde{\mbs{\omega}}\right)^\times
    \right)^\mathsf{V} \nonumber \\
    &= -\onehalf \left(
        \trace(\mbf{C}_{pd} ) \bone - \mbf{C}_{pd} 
    \right)   \tilde{\mbs{\omega}}, \label{eq:Lemma3_proof2}
\end{align}
where $\tilde{\mbs{\omega}} = \mbs{\omega}^{pd} = \mbs{\omega}^{pa} - \mbs{\omega}^{da}$. Substituting~\eqref{eq:Lemma3_proof2} into~\eqref{eq:Lemma3_proof1} and using~\eqref{lemma3:nud} yields
\bdis
\mbf{P} \mbf{s} = \bbm \dot{\tilde{\mbf{r}}}  \\ \tilde{\mbs{\omega}} \ebm + \mbf{P}\mbs{\Lambda} \tilde{\mbf{p}} = \tilde{\bnu} + \mbf{P}\mbs{\Lambda} \tilde{\mbf{p}} = \tilde{\mbs{\nu}}_r.
\edis
\end{proof}

% \rjc{It would be a good idea to comment on the invertibility of the term in $\mbf{P}$.}
The inverse of $\left(
        (\trace(\mbf{C}_{pd} ) \bone - \mbf{C}_{pd} ) 
    \right)$ in the definition of $\mbf{P}$ exists as long as $\trace(\mbf{C}_{pd}) \neq 1$. Similar to the case of the quaternion, this singularity is avoided as long as $\mathcal{F}_p$ and $\mathcal{F}_d$ are within a $\pm \pi/2$~rad rotation of each other, which is large enough to account for the wrench-feasible workspaces of most over-constrained CDPRs.%This is true for all possible workspace attitudes less than $\pi/2$ rad. 
% This is true because \tr(C) = 1+2\cos{\phi} where \phi is the axis rotation angle, essentially same range as quaternion.     

It is worth noting that the term $\mbf{W}$ in the adaptive feedforward-based control input of~\eqref{eq:tau_hat_ff} relies on the computation of $\dot{\mbs{\nu}}_r = \dot{\mbs{\nu}}_d - \left(\dot{\mbf{P}} \mbs{\Lambda} \tilde{\mbf{p}} + \mbf{P} \mbs{\Lambda} \dot{\tilde{\mbf{p}}}\right)$, which requires an expression for $\dot{\mbf{P}}$.  For the case of unconstrained attitude parameterizations, such as a 3-2-1 Euler-angle sequence, this involves simply taking the time derivative of $\mbf{S}(\mbf{q}^{pa})$.  For the quaternion, this computation is more involved, where $\dot{\mbf{P}}$ is solved for using the matrix product rule $\frac{\mathrm{d}}{\mathrm{d}t}\left(\mbf{A}^{-1}\right) = - \mbf{A}^{-1} \dot{\mbf{A}} \mbf{A}^{-1}$ to obtain
% \bdis
%     \dot{\mbf{P}} = \bbm \mbf{0} & \mbf{0} \\ \mbf{0} & -2( \delta \eta \bone + \delta \mbs{\epsilon}^\times )^{-1} ( \delta \dot{\eta} \bone + \delta \dot{\mbs{\epsilon}}^\times ) ( \delta \eta \bone + \delta \mbs{\epsilon}^\times )^{-1} \ebm,
% \edis
\bdis
    \dot{\mbf{P}} = \text{diag}\{\mbf{0},-2( \delta \eta \bone + \delta \mbs{\epsilon}^\times )^{-1} ( \delta \dot{\eta} \bone + \delta \dot{\mbs{\epsilon}}^\times ) ( \delta \eta \bone + \delta \mbs{\epsilon}^\times )^{-1}\},
\edis
where $\delta \dot{\eta}$ and $\delta \dot{\mbs{\epsilon}}$ are found by differentiating~\eqref{eq:quat_err} with respect to time.
% for using $\delta \dot{\mbf{q}} = \dot{\mbf{q}}^{pa} \otimes \mbf{q}_d^{-1} + \mbf{q}^{pa} \otimes \dot{\mbf{q}}_d^{-1}$. 
A similar procedure is used to compute $\dot{\mbf{P}}$ when using the $SO(3)$ description of attitude, where
\begin{align*}
    \dot{\mbf{P}} &= \text{diag}\{\mbf{0},-2\mbs{\Gamma}\left(\trace(\dot{\mbf{C}}_{pd}) \mbf{1} - \dot{\mbf{C}}_{pd}\right) \mbs{\Gamma}\} \\
    &= \text{diag}\{\mbf{0},-2\mbs{\Gamma}\left(-\trace(\mbstilde{\omega}^{\times} \mbf{C}_{pd}) \mbf{1} + \mbstilde{\omega}^{\times} \mbf{C}_{pd}\right) \mbs{\Gamma}\},
\end{align*}
$\mbs{\Gamma} = \left(\trace(\mbf{C}_{pd}) \mbf{1} - \mbf{C}_{pd}\right)^{-1}$, and Poisson's equation is used to simplify the expression for $\dot{\mbf{C}}_{pd}$.

\subsection{Passivity and Closed-Loop Stability Analyses}
\label{sec:Analysis}

%\subsection{Adaptive Feedforward Control}

\begin{theorem}
\label{theorem1}
   Consider a CDPR with error dynamics defined in~\eqref{eq:errordyn} and the adaptive feedforward-based control input of~\eqref{eq:TorqWithAdapFeedForward} with  the adaptive update law in~\eqref{eq:adapt_update}. Assuming that the dynamics of the CDPR are dominated by its rigid-body payload, the input-output mapping $ \bar{\mbf{f}}_{fb} \mapsto \mbf{s}$ is passive, where $\bar{\mbf{f}}_{fb} = \mbf{P}^\trans \mbf{f}_{fb}$.
\end{theorem}
\begin{proof}
% Substituting it into the passivity integral of Equation \ref{eq:passive1} yields
%
% \begin{align}
%     \int_0^T   \mbf{s}^\trans \mbf{P}^\trans \hat{\mbs{\tau}}_c  \cdot dt
%     + \int_0^T   \mbf{s}^\trans \mbf{P}^\trans \mbf{W} {\tilde{\mbf{a}}}  \cdot dt 
%     &\geq \beta
% \end{align}
% If the first integral is passive, then the mapping of $\mbf{P}^\trans \mbf{W} \tilde{\mbf{a}} \mapsto \mbf{s}$ is also passive if $\tilde{\mbf{a}}$ is chosen to be strictly passive such as 
% \begin{align}
%     \dot{\tilde{\mbf{a}}} &= 
%     \dot{\hat{\mbf{a}}} =
%     - \mbs{\Upsilon} \mbf{W}^\trans \tilde{\bnu}_r
%     \\ %%
%     \mbs{\Upsilon} &= \mbs{\Upsilon}^\trans > 0
% \end{align}
%
%
Define the non-negative function $$V_1 = 
    \onehalf \tilde{\bnu}_r^\trans \mbf{M} \tilde{\bnu}_r
    + \onehalf \tilde{\mbf{a}}^\trans \mbs{\Upsilon}^{-1} \tilde{\mbf{a}}.$$ 
    Taking the derivative of $V_1$, substituting in the adaptive update law and~\eqref{eq:errordyn} results in
\begin{align}
    \dot{V}_1 &=
    \tilde{\bnu}_r^\trans \mbf{M} \dot{\tilde{\bnu}}_r
    + \onehalf \tilde{\bnu}_r^\trans \dot{\mbf{M}} {\tilde{\bnu}}_r
    + \tilde{\mbf{a}}^\trans \mbs{\Upsilon}^{-1} \dot{\tilde{\mbf{a}}}
    \nonumber \\ %%
    &= \tilde{\bnu}_r^\trans \left(\mbf{W} {\tilde{\mbf{a}}} + \mbf{f}_{fb}\right)
    + \onehalf \tilde{\bnu}_r^\trans 
    (\dot{\mbf{M}} - 2 \mbf{D})
    \tilde{\bnu}_r - \tilde{\mbf{a}}^\trans\mbf{W}^\trans \tilde{\bnu}_r \nonumber\\
    % &= \tilde{\bnu}_r^\trans \tilde{\mbs{\tau}}  
    % + \tilde{\mbf{a}}^\trans \mbs{\Upsilon}^{-1} \dot{\tilde{\mbf{a}}} \nonumber
    % \\ %%
    &= \tilde{\bnu}_r^\trans \mbf{f}_{fb} %%
    %&=
    %\tilde{\bnu}_r^\trans \left(
    %\mbf{W} {\tilde{\mbf{a}}} + %{\hat{\mbs{\tau}}}_c
    %\right)
    %-    {\tilde{\mbf{a}}}^\trans \mbf{W}^\trans %\tilde{\bnu}_r
    %\\ %%
    =
    (\mbf{P} \mbf{s})^\trans \mbf{f}_{fb}  %%
    % &=
    % \mbf{s}^\trans \
    % \\ %%
    =
    \mbf{s}^\trans \bar{\mbf{f}}_{fb}.
    \label{eq:adaptive_passive}
\end{align}
%where $\bar{\mbf{f}}_{fb} = \mbf{P}^\trans \mbf{f}_{fb}$. 
Integrating~\eqref{eq:adaptive_passive} from $t=0$ to $t = T$, where $T \in \mathbb{R}_{\geq 0}$ gives $$\int_0^T \mbf{s}^\trans \bar{\mbf{f}}_{fb} \mathrm{d} t = V_1(T) - V_1(0) \geq -V_1(0),$$ 
% \bdis
%     %\label{eq:passive1}
%     \int_0^T \mbf{s}^\trans \bar{\mbs{\tau}}_{fb} \mathrm{d} t = V_1(T) - V_1(0) \geq -V_1(0),
% \edis
which proves the mapping $ \bar{\mbf{f}}_{fb} \mapsto \mbf{s}$ is passive. 
\end{proof}

\begin{figure}[t!]
\centering
    \includegraphics[width=0.49\textwidth]{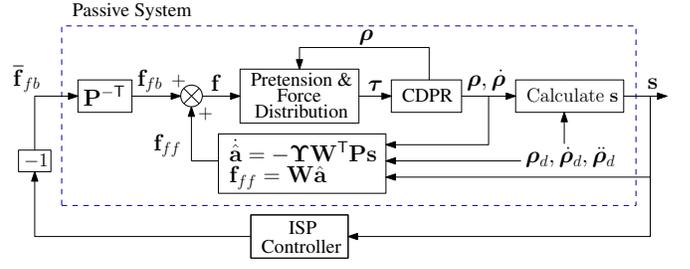}
\centering
%\vspace{-10pt}
\caption{Block diagram of a CDPR with an adaptive feedforward-based control input and pretension \& force distribution satisfying $\mbf{f} = \mbs{\Pi}^\trans(\mbs{\rho})\mbs{\tau}$
% ~\eqref{eq:ControlAlloc},
which is proven to be passive in Theorem~\ref{theorem1}. The passive system is in negative feedback with an ISP controller.}%, which is chosen as an SPR controller with feedthrough.}
	\label{fig:IO_Stable}
\vspace{-10pt}
\end{figure}

\begin{corollary}
\label{corollary1}
The closed-loop system involving the CDPR with error dynamics defined in
% ~\eqref{eq:EoMTask}
~\eqref{eq:errordyn}
, the adaptive feedforward-based control input of~\eqref{eq:TorqWithAdapFeedForward}, and an ISP negative feedback controller (or alternatively a linear time-invariant (LTI) strictly positive real (SPR) negative feedback controller) is input-output stable (i.e., $\mbf{s} \in \mathcal{L}_2$).
\end{corollary}
\begin{proof}
Knowing that the input-output mapping $\bar{\mbs{\tau}}_{fb} \mapsto \mbf{s}$ is passive and an ISP controller is implemented in a negative feedback connection with this mapping (see Fig.~\ref{fig:IO_Stable}), the passivity theorem guarantees that $\mbf{s} \in \mathcal{L}_2$~\cite[p.~358]{Brogliato2020}. In the case of an SPR feedback controller, Theorem~8.10 in~\cite[p.~219]{MarquezBook} can be used to obtain the same result.
\end{proof}

Corollary~\ref{corollary1} guarantees closed-loop input-output stability with the use of any ISP or SPR feedback controller. This result does not rely on exact knowledge of the parameters of the CDPR's dynamics, and thus, robust closed-loop input-output stability is guaranteed. However, it is worth noting that robustness to pose estimation error is not guaranteed and falls beyond the scope of this work. Although there are a number of ISP controllers that can be used to ensure closed-loop input-output stability, an SPR controller with transfer matrix $\mbf{G}_c(s) = \mbf{C}_c \left(s\mbf{1} - \mbf{A}_c\right)^{-1} \mbf{B}_c$ is considered in this paper. The SPR property of $\mbf{G}_c(s)$ ensures that there exist $\mbf{P}_c = \mbf{P}_c^\trans > 0$ and $\mbf{Q}_c = \mbf{Q}_c^\trans > 0$ such that~\cite[p.~93]{Brogliato2020}
\begin{align*}
    \mbf{P}_c \mbf{A}_c + \mbf{A}_c^\trans \mbf{P}_c &=- \mbf{Q}_c, \\
    \mbf{P}_c \mbf{B}_c &= \mbf{C}_c^\trans.
\end{align*}
The feedback control input is then chosen as $\mbf{f}_{fb} = - \mbf{P}^{-\trans} \mbf{y}_c$, which results in $\bar{\mbf{f}}_{fb} = \mbf{P}^\trans \mbf{f}_{fb} =  -\mbf{P}^\trans\left(\mbf{P}^{-\trans} \mbf{y}_c\right) = - \mbf{y}_c$, where $\mbf{y}_c(s) = \mbf{G}_c(s) \mbf{s}(s)$. 
% \begin{lemma}
% \label{lemma:SPR_ISP}
% The controller defined in~\eqref{eq:SPR}, where $\mbf{G}_c(s) = \mbf{C}_c \left(s\mbf{1} - \mbf{A}_c\right)^{-1} \mbf{B}_c$ is an SPR transfer function, has an ISP mapping $\mbf{s} \mapsto \mbf{y}_c$.
% \end{lemma}
% \begin{proof}
% Define the non-negative function $V_2 = \onehalf \mbf{x}_c^\trans \mbf{P}_c \mbf{x}_c$, where $\mbf{P}_c = \mbf{P}_c^\trans > 0$. Taking the time derivative of $V_2$ and substituting in the properties of an SPR system results in
% \begin{align}
%     \dot{V}_2 &= \onehalf\mbf{x}_c^\trans \left( \mbf{P}_c\mbf{A}_c + \mbf{A}_c^\trans \mbf{P}_c\right) \mbf{x}_c + \mbf{x}_c^\trans \mbf{P}_c \mbf{B}_c \mbf{s} \nonumber \\
%     &= - \onehalf \mbf{x}_c^\trans \mbf{Q}_c \mbf{x}_c + \mbf{x}_c^\trans \mbf{C}_c^\trans \mbf{s} \nonumber \\
%     &\leq -\onehalf \lambda_{\textrm{min}}(\mbf{Q}_c) \mbf{x}_c^\trans \mbf{x}_c + \left(\mbf{y}_c - \delta \mbf{s}\right)^\trans \mbf{s} \label{eq:ISPproof1}
% \end{align}
% Rearranging~\eqref{eq:ISPproof1}, integrating from $t = 0$ to $t = T$, and knowing that $V_2(T) \geq 0$ and $\mbf{Q}_c = \mbf{Q}_c^\trans > 0$ gives $\int_{0}^T \mbf{y}_c^\trans \mbf{s} \mathrm{d}t \geq -\delta \norm{\mbf{s}}_{2T} - V_2(0)$, 
% % \bdis
% % \int_{0}^T \mbf{y}_c^\trans \mbf{s} \mathrm{d}t \geq -\delta \norm{\mbf{s}}_{2T} - V_2(0),
% % \edis
% which proves that $\mbf{s} \mapsto \mbf{y}_c$ is an ISP mapping.
% \end{proof}

\begin{theorem}
\label{theorem:stability}
The control law in~\eqref{eq:TorqWithAdapFeedForward} and~\eqref{eq:tau_hat_ff}, where $\mbf{f}_{fb} = -\mbf{P}^{-\trans}\mbf{y}_c$ and $\mbf{y}_c$ is the output of an SPR controller with input $\mbf{s}$, ensures asymptotic convergence of the pose tracking and rate tracking errors (i.e., $\tilde{\mbf{p}} \to \mbf{0}$ and $\tilde{\mbs{\nu}} \to \mbf{0}$ as $t \to \infty$), when applied to the CDPR with dynamics given by~\eqref{eq:EoMTask}.
\end{theorem}
\begin{proof}
From Corollary~\ref{corollary1}, it is known that $\mbf{s} \in \mathcal{L}_2$. Rearranging~\eqref{eq:svar} yields $\dot{\mbftilde{p}} = - \mbs{\Lambda} \mbftilde{p} + \mbf{s}$, which is an asymptotically stable LTI system whose input is in $\mathcal{L}_2$. This results in $\mbftilde{p} \in \mathcal{L}_2 \cap \mathcal{L}_\infty$, $\dot{\mbftilde{p}} \in \mathcal{L}_2$, and $\mbftilde{p} \to \mbf{0}$ as $t \to \infty$~\cite[p.~269]{Brogliato2020}. 
To prove that $\mbstilde{\nu} \to \mbf{0}$ as $t \to \infty$, define the non-negative function 
\bdis
V_2 = V_1 + \mbf{x}_c^\trans \mbf{P}_c \mbf{x}_c,
\edis
where $\mbf{P}_c = \mbf{P}_c^\trans > 0$. Making use of the SPR property of the feedback controller, the time derivative of $V_2$ is
\begin{align}
    \dot{V}_2 &= -\mbf{s}^\trans\mbf{y}_c + \mbf{x}_c^\trans \left( \mbf{P}_c\mbf{A}_c + \mbf{A}_c^\trans \mbf{P}_c\right) \mbf{x}_c + \mbf{x}_c^\trans \mbf{P}_c \mbf{B}_c \mbf{s} \nonumber \\
    &\leq -\mbf{s}^\trans\mbf{y}_c -  \mbf{x}_c^\trans \mbf{Q}_c \mbf{x}_c + \mbf{x}_c^\trans\mbf{C}_c^\trans \mbf{s} \nonumber \\
    &\leq -\mbf{s}^\trans\mbf{y}_c - \lambda_{\textrm{min}}(\mbf{Q}_c) \mbf{x}_c^\trans \mbf{x}_c + \mbf{y}_c^\trans \mbf{s} \nonumber \\
    &\leq - \lambda_{\textrm{min}}(\mbf{Q}_c) \mbf{x}_c^\trans \mbf{x}_c \leq 0. \label{eq:proof3a}
\end{align}
Integrating~\eqref{eq:proof3a} from $t = 0$ to $t = T$ results in $V_2(T) \leq V_2(0)$, which proves that $\{\mbstilde{\nu}_r,\mbftilde{a},\mbf{x}_c\} \in \mathcal{L}_\infty$. Through the relationship $\mbf{s} = \mbf{P}^{-1} \mbstilde{\nu}_r$, where $\mbf{P}^{-1}$ is bounded, it is known that $\mbf{s} \in \mathcal{L}_\infty$. This also results in $\dot{\mbftilde{p}} \in \mathcal{L}_\infty$, since $\dot{\mbftilde{p}} = - \mbs{\Lambda} \mbftilde{p} + \mbf{s}$. %Knowing that $\mbf{a} \in \mathcal{L}_\infty$ (i.e., the parameters that define the CDPR's dynamics are bounded) ensures that $\mbfhat{a} = \mbftilde{a} + \mbf{a} \in \mathcal{L}_\infty$. 
Assuming that $\mbs{\nu}_d \in \mathcal{L}_\infty$ and $\mbf{P}$ is bounded, $\mbftilde{p} \in \mathcal{L}_\infty$ ensures that $\mbs{\nu}_r \in \mathcal{L}_\infty$ through~\eqref{eq:FilteredRates}. Taking the time derivative of~\eqref{eq:FilteredRates} yields $\dot{\mbs{\nu}}_r = \dot{\mbs{\nu}}_d - \left( \dot{\mbf{P}} \mbs{\Lambda} \mbftilde{p} + \mbf{P} \mbs{\Lambda} \dot{\mbftilde{p}}\right)$. Assuming that $\dot{\mbs{\nu}}_r \in \mathcal{L}_\infty$ and $\dot{\mbf{P}}$ is bounded, $\{\mbftilde{p},\dot{\mbftilde{p}}\} \in \mathcal{L}_\infty$ ensures $\dot{\mbs{\nu}}_r \in \mathcal{L}_\infty$. With $\{\mbs{\nu}_r,\dot{\mbs{\nu}}_r,\mbftilde{a},\mbf{x}_c\} \in \mathcal{L}_\infty$, $\mbf{f}-\mbf{f}_d = \mbf{W}\mbftilde{a} - \mbf{C}_c \mbf{x}_c \in \mathcal{L}_\infty$. Through the error dynamics of~\eqref{eq:errordyn}, $\{\mbstilde{\nu}_r,\mbf{f}-\mbf{f}_d\} \in \mathcal{L}_\infty$ results in $\dot{\mbstilde{\nu}}_r \in \mathcal{L}_\infty$. Knowing that $\mbf{s} \in \mathcal{L}_2$, the relationship $\mbstilde{\nu}_r = \mbf{P} \mbf{s}$ leads to $\mbstilde{\nu}_r \in \mathcal{L}_2$. Barbalat's lemma can be used to prove $\mbstilde{\nu}_r \to \mbf{0}$ as $t \to \infty$, since $\mbstilde{\nu}_r \in \mathcal{L}_2$ and $\dot{\mbstilde{\nu}}_r \in \mathcal{L}_\infty$~\cite[p.~657]{Brogliato2020}. It then follows that $\mbstilde{\nu} \to \mbf{0}$ as $t \to \infty$, since $\mbstilde{\nu} = \mbstilde{\nu}_r - \mbf{P} \mbs{\Lambda} \mbftilde{p}$ and both $\mbstilde{\nu}_r \to \mbf{0}$ and $\mbftilde{p} \to \mbf{0}$ as $t \to \infty$.
%
%
%The filtered error of ~\eqref{eq:svar} $\mbf{s} \to \mbf{0}$ as $t \to \infty$ by Barbalat's Lemma. This implies $\tilde{\mbf{p}} \to \mbf{0}$ as $t \to \infty$.
\end{proof}

Theorem~\ref{theorem:stability} demonstrates that the proposed control law ensures that $\mbftilde{r} \to \mbf{0}$ and $\mbftilde{p} \to \mbf{0}$ as $t \to \infty$. This results in the position of the CDPR's payload satisfying $\mbf{r} \to \mbf{r}_d$ as $t \to \infty$. The interpretation of $\mbftilde{p} \to \mbf{0}$ as $t \to \infty$ depends on the chosen attitude parameterization: $\mbf{q}^{pa} \to \mbf{q}^{da}$ for unconstrained attitude parameterizations, $\delta \mbs{\epsilon} \to \mbf{0}$ (equivalent to $\mbf{q}^{pa} \to \pm \mbf{q}_d$) for the quaternion, and $\mbf{C}_{pa} \to \mbf{C}_{da}$ for $SO(3)$, all of which describe asymptotic convergence of the attitude of the CDPR's payload to the desired attitude. Note that as in~\cite{Ortega1989,slotine1987adaptive,Damaren1996}, there is no guarantee that $\mbfhat{a} \to \mbf{a}$ as $t \to \infty$, as $\mbfhat{a}$ evolves in a manner that only guarantees asymptotic tracking of the desired payload pose.

\section{CDPR Numerical Example}
\label{sec:Examples}

Consider a $6$~DOF CDPR with $m=8$ cables and a rigid-body payload, as shown in Fig.~\ref{fig:CDPR_Anim}, with %A CDPR with rigid cables is used in this numerical example for simplicity.
numerical %values for the CDPR are 
%chosen based on the system that is being built at the Aerospace, Robotics, Dynamics, and Control (ARDC) Lab, University of Minnesota. The specifications are presented 
provided in Table~\ref{Table:CDPR_parameters}. The locations of the 8 stationary winches and the attachment points of the cables on the rigid-body payload are given in Table~\ref{Table:cable_loc}. A crossed-cable configuration similar to the IPAnema~2 setup described in~\cite[p.~319]{Pott2018} is used, which results in a relatively large wrench-feasible translational and rotational workspace while avoiding cable collisions.

The numerical simulation is fashioned from the Lagrangian-based dynamic model developed in~\cite{Godbole2019} for flexible cables whose mass and stiffness properties vary with the length of the cable, and is extended to accommodate a 6~DOF, 8-cable CDPR. A first set of simulations is performed with cables modeled as rigid straight lines, where the elastic coordinates of the model from~\cite{Godbole2019} are constrained to be zero (i.e., no elastic deformation can occur). The second set of simulations models elastic deformation of the cables in the axial and two transverse directions with the Rayleigh-Ritz method in~\cite{Godbole2019}. 
%
%Two cable models are considered in this study. The first model considers the cables as rigid straight lines, while the second model includes cable flexibility using Ritz basis functions in a similar fashion to~\cite{Godbole2019}. The three-dimensional elastic deformation of the flexible cables is included in the axial and two transverse directions. %, and out-of-plane directions. 
Both numerical models include cable mass and allow the cables to transmit forces only when under tension. An aramid cable with properties listed in Table \ref{Table:CDPR_parameters} is used. % and are fashioned after an Aramid cable. 
In the case with flexible cables, the pose of the payload used by the controller is computed through forward kinematics~\cite{nguyenAttitude} using only the rigid rotation of the winches in order to simulate a realistic implementation scenario and demonstrate robustness to imperfect knowledge of the payload pose. % The iterative forward kinematics algorithm in~\cite{nguyenAttitude} is used for this computation. 
All pose tracking errors in the result plots are of the actual payload pose, computed using the deformed cables.% basis functions to model flexibility and extension to 3 degrees-of-freedom relative to the frame of the cable $\mathcal{F}_c$.% The payload, winch and cable dynamics were constrained together using the null space method~\cite{CaverlyNullSpace}.

\begin{table}[!t]
\caption{CDPR parameters used in the numerical simulation. \label{Table:CDPR_parameters}}
\centering
\begin{tabular}{|c||c|}
\hline
Parameter & Value   \\\hline\hline
Payload mass~(kg) & $m_p = 6.75$ \\\hline
Payload inertia~(g$\cdot$m$^2$) & $\mbf{I}_p = \text{diag} \{15.8,5.2,14.7\}$ \\ \hline
% Payload inertia & $\mbf{I}_p = \text{diag} \{I_{xx},I_{yy},I_{zz}\}$, \\ 
% (g$\cdot$m$^2$) & $I_{xx}=15.8$, $I_{yy}=5.2$, $I_{zz}=14.7$ \\\hline
Cable density~(g/m) & $\rho = 4.6$ \\\hline
Cable elasticity~(N/m$^2$) & $E = 127 \times 10^9$ \\\hline
Cable Radius~(mm) & $r_c = 1$ \\\hline
Winch radius~(m) & $r_i = 0.0254$, $i=1,\ldots,8$ \\\hline
Winch inertia~(g$\cdot$m$^2$) & $J_i = 0.025$, $i=1,\ldots,8$ \\\hline
\end{tabular}
\end{table}

\begin{table}[!t]
\caption{Cable attachment points in the numerical simulation. \label{Table:cable_loc}}
\centering
\begin{tabular}{|c||c|c|}
\hline
Cable& Winch Position Rel. to& Payload Attachment Rel. to\\ & Origin in $\mathcal{F}_a$ (cm)& Payload CoM in $\mathcal{F}_p$ (cm)\\\hline\hline
1 & $ \bbm 71 & 38 & 93 \ebm$ & $ \bbm 3 & 7.5 & -3.75 \ebm$    \\ \hline
2 & $ \bbm -71 & 38 & 93 \ebm $ & $ \bbm -3& 7.5 & -3.75\ebm$   \\ \hline
3 & $ \bbm -71 & -38 & 93 \ebm$ & $ \bbm -3& -7.5 &-3.75\ebm$  \\ \hline
4 & $ \bbm 71 & -38 & 93 \ebm $ & $ \bbm 3& -7.5 & -3.75\ebm$   \\ \hline
5 & $ \bbm -71 & -38 & 0 \ebm$ & $ \bbm -1.5& -7.5& 3.75\ebm$ \\ \hline
6 & $ \bbm 71 & 38 & 0 \ebm$ & $ \bbm 1.5& 7.5& 3.75\ebm$   \\ \hline
7 & $ \bbm -71 & 38 & 0 \ebm$ & $ \bbm -1.5& 7.5& 3.75\ebm$  \\ \hline
8 & $ \bbm 71 & -38 & 0 \ebm$ & $ \bbm 1.5 & -7.5& 3.75\ebm$  \\ \hline 
\end{tabular}
%\vspace{-10pt}
\end{table}

% The flexible cables are modeled similarly to~\cite{Godbole2019} using Ritz basis functions to model flexibility and extension to 3 degrees-of-freedom relative to the frame of the cable $\mathcal{F}_c$. The payload, winch and cable dynamics were constrained together using the null space method~\cite{CaverlyNullSpace}. 

The desired payload position trajectory is $\mbf{r}_d^\trans= 0.1 [ \cos{(0.6\pi t)} \,\, \sin{(0.6\pi t)} \,\, \cos{(0.6\pi t)+4.65}  ]$~m and the desired payload attitude trajectory is $\mbf{q}^{da^\trans} = 20 [ \cos{(0.4\pi t - \pi/2)} \,\, \cos{(0.4\pi t - \pi/4)} \,\, \cos{(0.4\pi t)} ]$~deg, described in terms of a 3-2-1 Euler-angle sequence. Note that while an Euler-angle sequence is used here to define the desired attitude trajectory, any attitude parameteriztion can be used for this purpose and converted to the attitude parameterization chosen for the controller.

%It is emphasized that while a 3-2-1 Euler-angle sequence was selected here for easier visualization, the desired attitude could be expressed in any attitude parameterizations. 
% In fact, a translation was performed to their respective desired attitude parameterizations for computation. 

% The same attitude was expressed in other attitude parameterizations. 

Numerical simulations are performed with the proposed adaptive control law using various payload attitude parameterizations, including a 3-2-1 Euler angle sequence, $SO(3)$ (the DCM), the quaternion, the rotation vector, and MRPs. As a comparison, two simplifications of the proposed controller with a 3-2-1 Euler angle sequence are also tested in simulation, where small Euler angles are assumed (i.e., $\mbs{\omega}^{pa} \approx \dot{\mbf{q}}^{pa}$) in either the feedback controller and the adaptive feedforward-based controller (denoted as Simplified Euler) or only the feedback controller (denoted as Simplified FB Euler). These simplifications are similar to the linear control design and analysis performed in~\cite{santos2020redundancy}.

The negative feedback controller is implemented as $\mbf{f}_{fb} = -\mbf{P}^{-\trans} \mbf{y}_c$, where $\mbf{y}_c$ is the output of an SPR controller with feedthrough and input $\mbf{s}$. A variety of methods can be used to design an SPR controller (e.g., see~\cite{SPR}) and in this work a simple first-order low-pass filter $$\mbf{y}_c(s) = \mbf{K}_d \text{diag}\Bigl\{\frac{\omega_c}{s + \omega_c},\cdots,\frac{\omega_c}{s + \omega_c} \Bigr\} \mbf{s}(s),$$ where
%$\hat{\mbs{\tau}}_c = -\mbf{P}^{-\trans} \mbf{K}_d \frac{\omega_c}{s + \omega_c} \mbf{s}$ where 
$\mbf{K}_d = \mbf{K}_d^\trans > 0$ is the derivative gain, $\omega_c = 2\pi$~rad/s is the chosen cut-off frequency. The inertia entries of %the adaptive feedforward-based parameters of 
$\hat{\mbf{a}}$ are all initially set to zero, %except for the gravitational term in the vertical direction. It is assumed that 
while the payload mass is assumed to be approximately known and therefore $\hat{m}_p$ is initialized within $20$~\% of the true payload mass. % is chosen to intialize }. %gravitational force in the vertical direction is used for this term.}% \rev{The adaptive feedforward-based parameters does not need to converge to the actual physical values.} %was introduced as the initial estimate to provide sufficient feed forward payload forces to counter gravitational force.   
%
%
%In summary, the forces in task space $\hat{\mbs{\tau}} = -\mbf{P}^{-1} \mbf{K}_d \frac{\omega_c}{s + \omega_c} \mbf{s} + \mbf{W} \hat{\mbf{a}}$ where $\dot{\hat{\mbf{a}}} = - \mbs{\Upsilon} \mbf{W}^{\trans} \tilde{\bnu}_r$ and $\mbf{s}$ depends on the attitude parameterization selected. 
The control parameters used for the rigid cable simulation are $\mbs{\Lambda} = 10 \cdot \bone$, $\mbs{\Upsilon} = 5\cdot \bone$, $\mbf{K}_d = \text{diag}\{\mbf{K}_{d,v}, \mbf{K}_{d,\omega}\}$, $\mbf{K}_{d,v}=125\cdot \bone$, and $\mbf{K}_{d,\omega}= 16 \frac{2}{3} \cdot \bone$.
% $\mbf{K}_{d} = \mathrm{diag}(125,125,125,16 \frac{2}{3},16 \frac{2}{3}, 16 \frac{2}{3})$. 
For the quaternion-based controller, $\mbf{K}_{d}$ is doubled to ensure the control gain is the same for small angles across all attitude parameterizations and a fair performance comparison can be made (i.e., $\delta \mbs{\epsilon} \approx \onehalf \mbf{q}^{pa}$, where $\mbf{q}^{pa}$ is an unconstrained attitude parameterization). 
% $\mbf{K}_{d} = \mathrm{diag}(125,125,125,33 \frac{1}{3}, 33 \frac{1}{3}, 33 \frac{2}{3})$
%$\mbf{K}_{d,\omega} = 33 \frac{1}{3}$ was used for fair performance comparison to other attitude parameterizations. 
The control gains when simulating the CDPR with flexible cables are reduced to increase robustness to the unmodeled dynamics. Specifically, the terms $\mbf{K}_{d}$ and $\mbs{\Lambda}$ are reduced by a factor of 5 and 2, respectively. This is a common strategy used when controlling the motion of a CDPR with flexible cables~\cite{Caverly2015_TCST}.

\begin{figure}[t!]
\centering
    \includegraphics[width=0.49\textwidth]{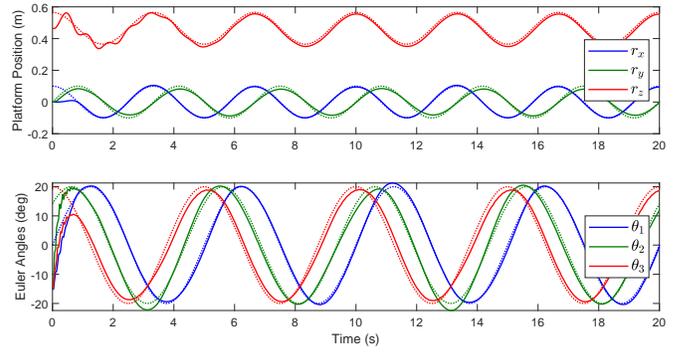}
\centering
%\vspace{-10pt}
\caption{Payload pose trajectory versus time with the $SO(3)$-based controller simulated with flexible cables: actual pose (solid) and desired pose (dashed). The payload attitude is expressed in terms of a 3-2-1 Euler-angle sequence only for visualization purposes.}
	\label{fig:Path}
%\vspace{-10pt}
\end{figure}

\begin{figure}[t!]
\centering
    \includegraphics[width=0.45\textwidth]{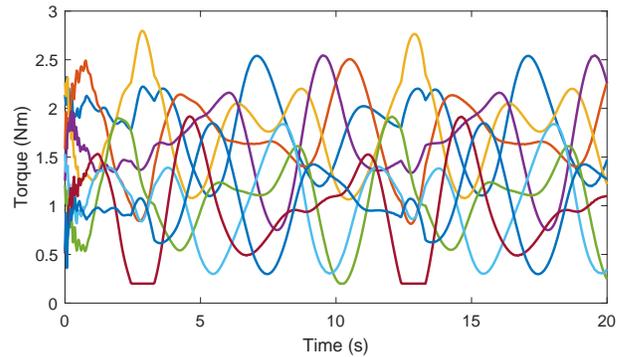}
\centering
%\vspace{-10pt}
\caption{Winch torques versus time for the simulation with the $SO(3)$-based controller and flexible cables.}
	\label{fig:Forces}
%\vspace{-10pt}
\end{figure}

\begin{figure}[t!]
\centering
    \includegraphics[width=0.49\textwidth]{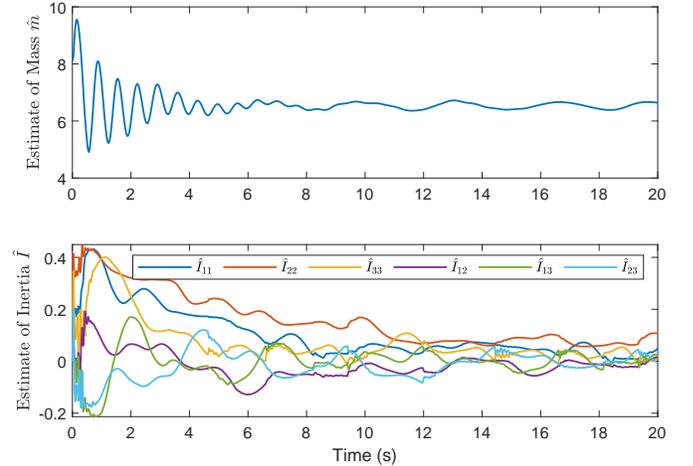} % *2 has horizontal legend
\centering
%\vspace{-10pt}
\caption{Estimated parameters $\mbfhat{a}$ versus time for the simulation with the $SO(3)$-based controller and flexible cables. The two subplots are the mass parameter (denoted $\hat{m}$) and inertial parameters (denoted $\hat{I}_{ij}$).}
% \caption{Estimated parameters $\mbfhat{a}$ versus time for the simulation with the $SO(3)$-based controller and flexible cables. The three subplots include the parameters of $\mbfhat{a}$ associated with the mass matrix (denoted $\hat{M}$), the nonlinear term (denoted $\hat{D}$), and the gravitational term (denoted $\hat{g}$), respectively.}
	\label{fig:FeedForwardPar}
%\vspace{-10pt}
\end{figure}

\begin{figure}[t!]
\vspace{7pt}
     \centering
     \subfigure[]{
         \centering
        \includegraphics[width=0.47\textwidth]{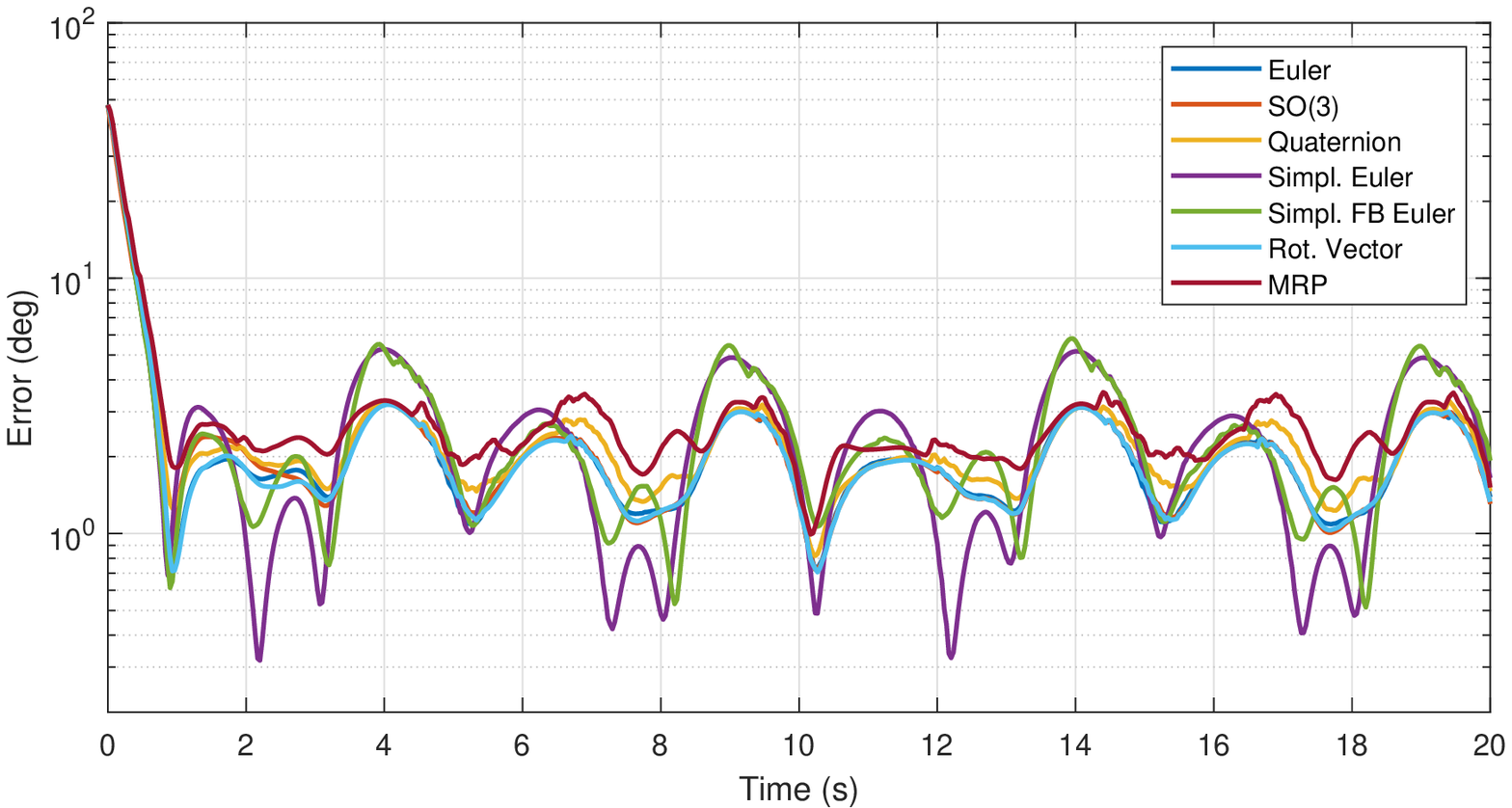}
         \label{fig:AxisAngleError}
     }
     \subfigure[]{
         \centering
        \includegraphics[width=0.45\textwidth]{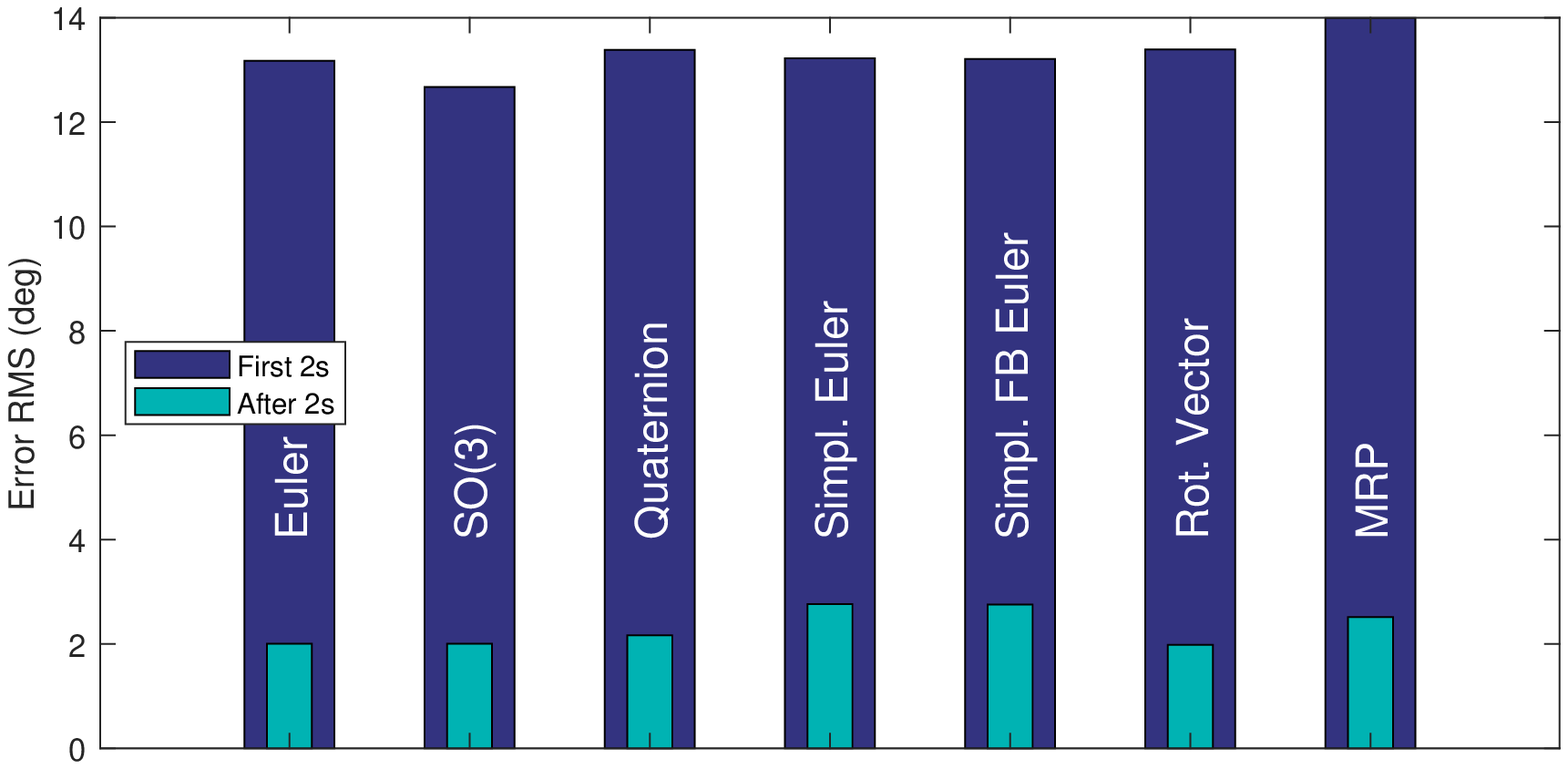}
         \label{fig:rmsError}
     }
     %\vspace{-8pt}
        \caption{(a) Attitude tracking error versus time and (b) RMS attitude tracking error with controllers based on different attitude parameterizations for the case of flexible cables. The error is the angle portion of an axis-angle parameterizaion of the attitude tracking error.}
        \label{fig:Errors}
       % \vspace{-10pt}
\end{figure}

\begin{figure*}[t!]
     \centering
     \subfigure[]{
     \centering
        \includegraphics[width=0.31\textwidth]{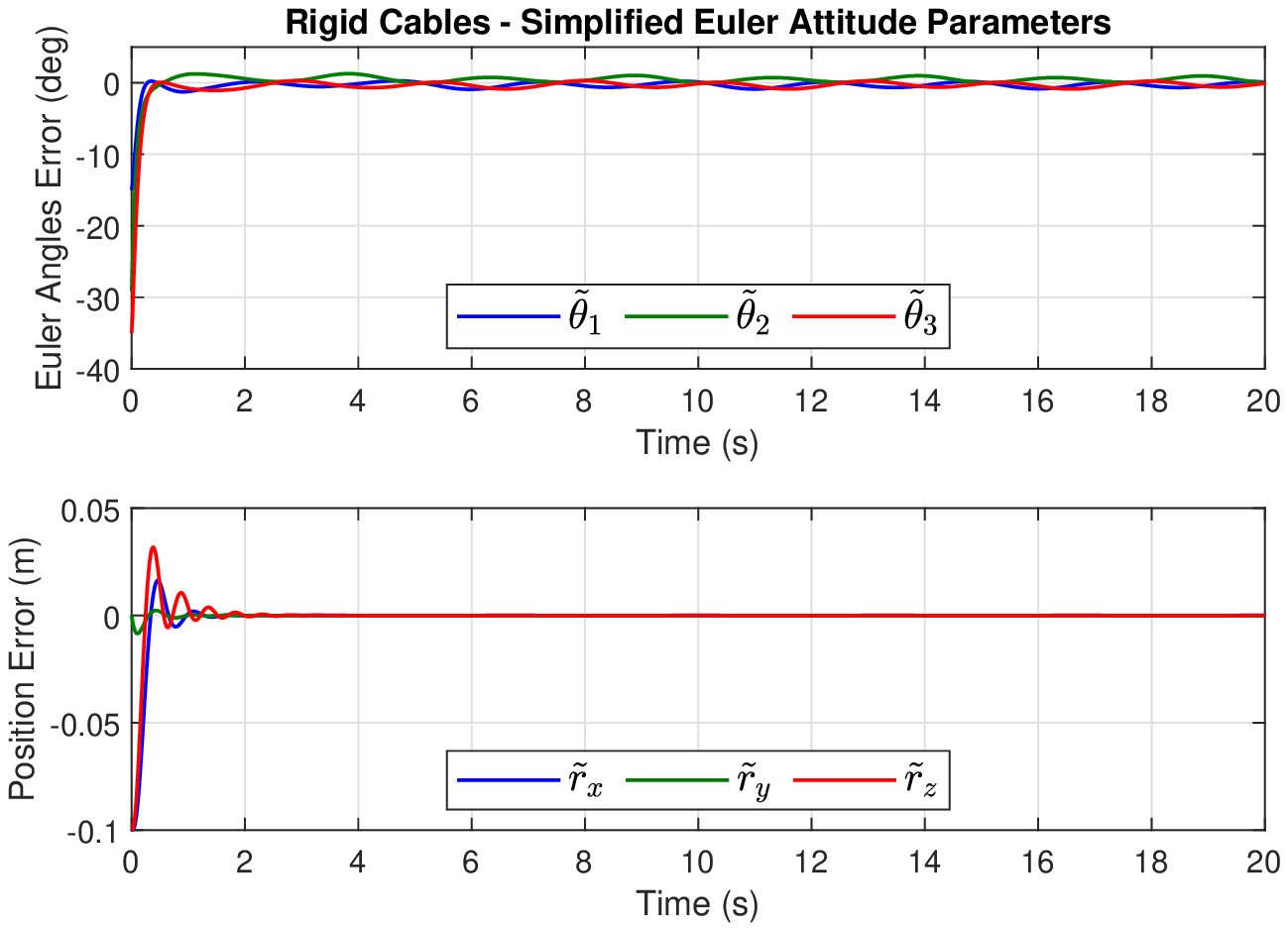}
            %\caption{}
         \label{fig:Rigid_BadEuler_ang}
     }
     %\hfill
     \subfigure[]{
         \centering
        \includegraphics[width=0.31\textwidth]{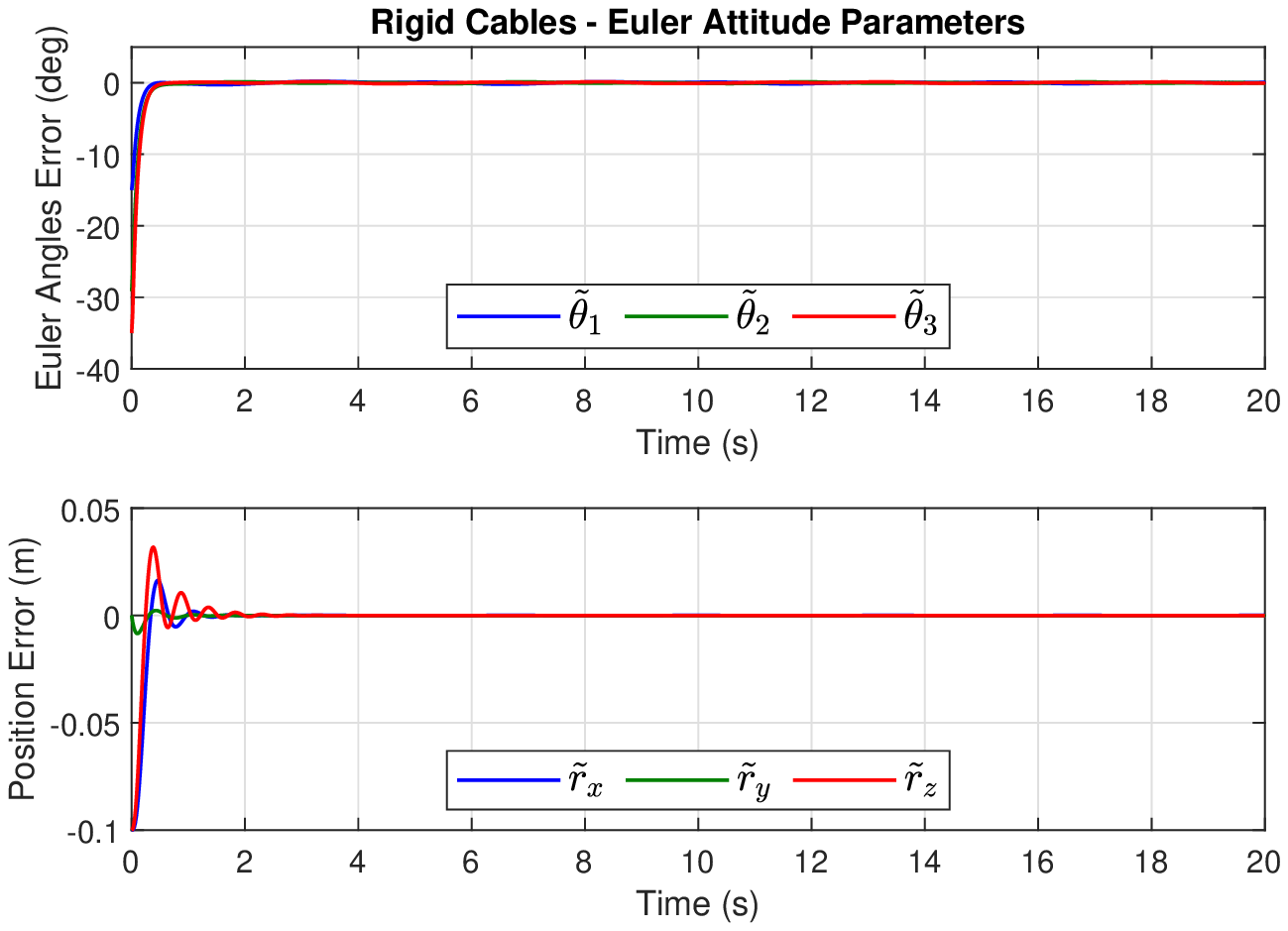}
            % \caption{}
         \label{fig:Rigid_GoodEuler_angx}
     }
     %\hfill
     \subfigure[]{
         \centering
        \includegraphics[width=0.31\textwidth]{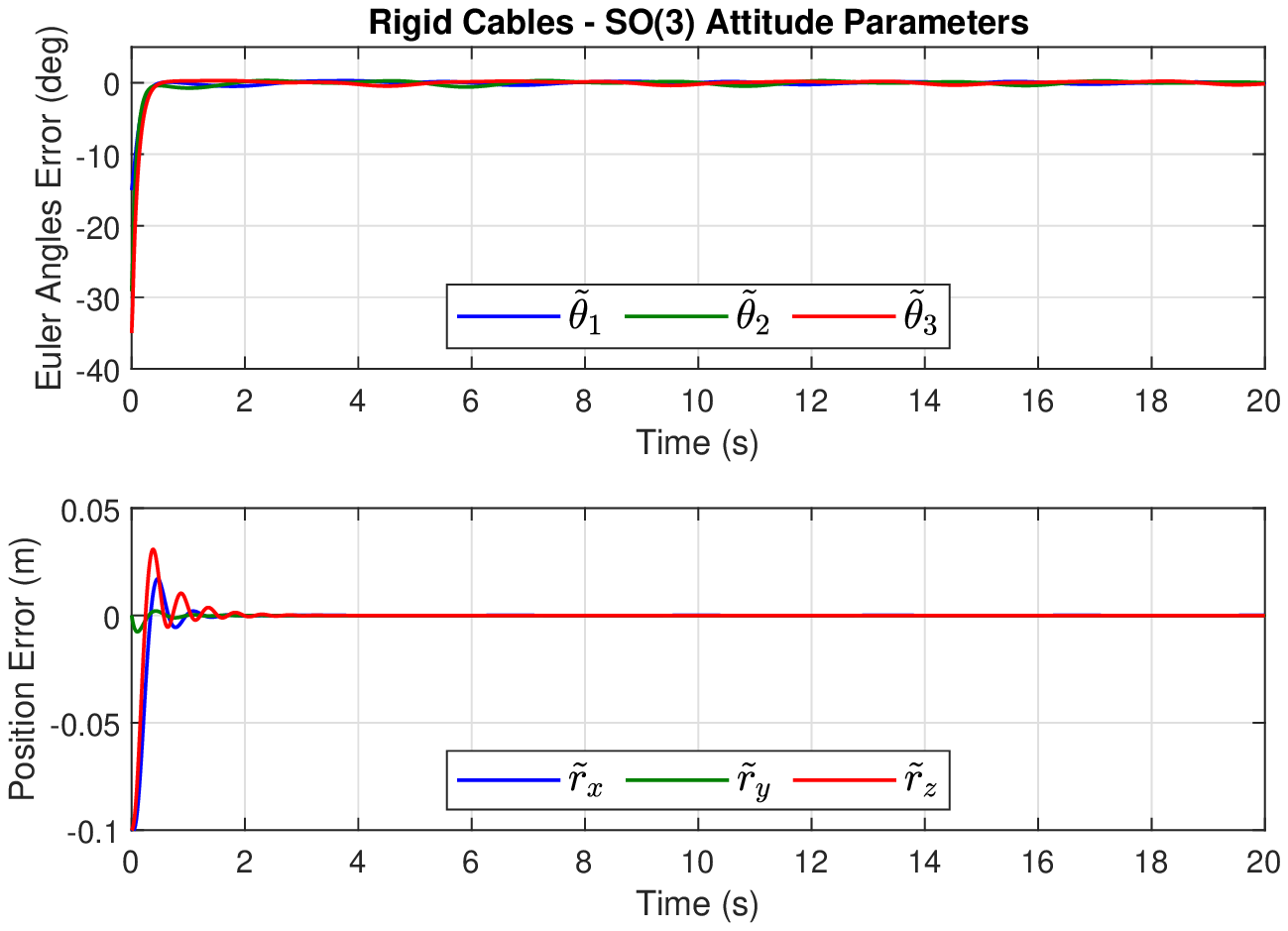}
        %  \caption{$y=5/x$}
         \label{fig:Rigid_SO3_ang}
     }
     %\\
    %  \subfigure[]{
    %      \centering
    %      \includegraphics[width=0.31\textwidth]{Figures2/Rigid_BadEuler_pos.eps}
    %     %  \caption{$y=x$}
    %      \label{fig:Rigid_BadEuler_pos}
    %  }
    %  %\hfill
    %  \subfigure[]{
    %      \centering
    %      \includegraphics[width=0.31\textwidth]{Figures2/Rigid_GoodEuler_pos.eps}
    %     %  \caption{$y=3sinx$}
    %      \label{fig:Rigid_GoodEuler_posx}
    %  }
    %  %\hfill
    %  \subfigure[]{
    %      \centering
    %      \includegraphics[width=0.31\textwidth]{Figures2/Rigid_SO3_pos.eps}
    %     %  \caption{$y=5/x$}
    %      \label{fig:Rigid_SO3_pos}
    %  }
    % \vspace{-8pt}
        \caption{Payload pose tracking errors in the rigid-cable simulations with (a) the simplified Euler-angle-based controller, (b) the correctly-implemented Euler-angle-based controller, and (c) the $SO(3)$-based controller. For visualization purposes, the attitude errors are plotted using a 3-2-1 Euler-angle sequence (denoted $\tilde{\theta}_1$, $\tilde{\theta}_2$, and $\tilde{\theta}_3$). The position errors in the three axes of $\mathcal{F}_a$ are denoted as $\tilde{r}_x$, $\tilde{r}_y$, and $\tilde{r}_z$.}
        \label{fig:RigidCableResponse}
\end{figure*}

\begin{figure*}[t!]
     \centering
     \subfigure[]{
         \centering
        \includegraphics[width=0.31\textwidth]{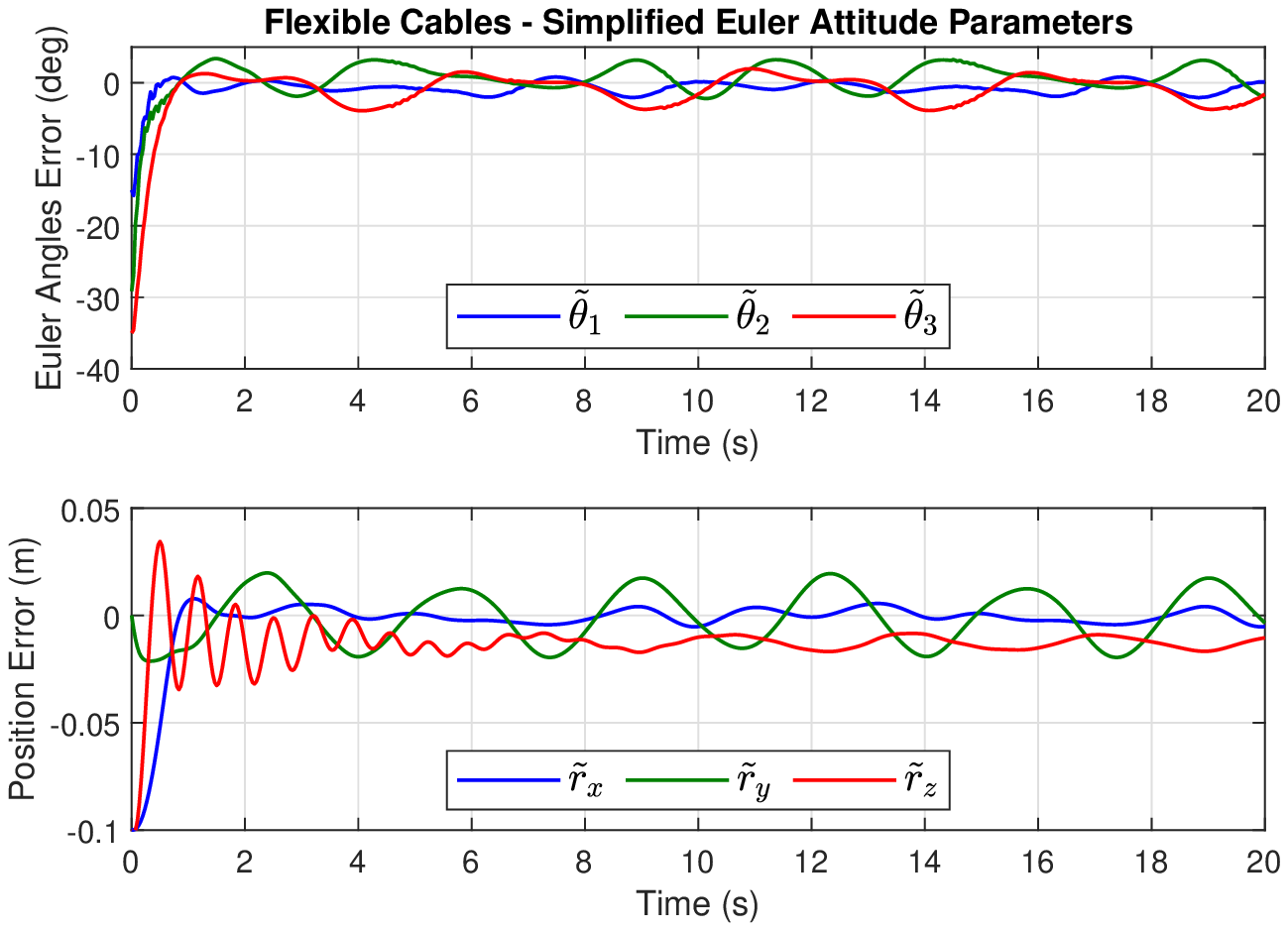}
        %  \caption{$y=x$}
         \label{fig:Flex_BadEuler_ang}
     }
     \subfigure[]{
         \centering
        \includegraphics[width=0.31\textwidth]{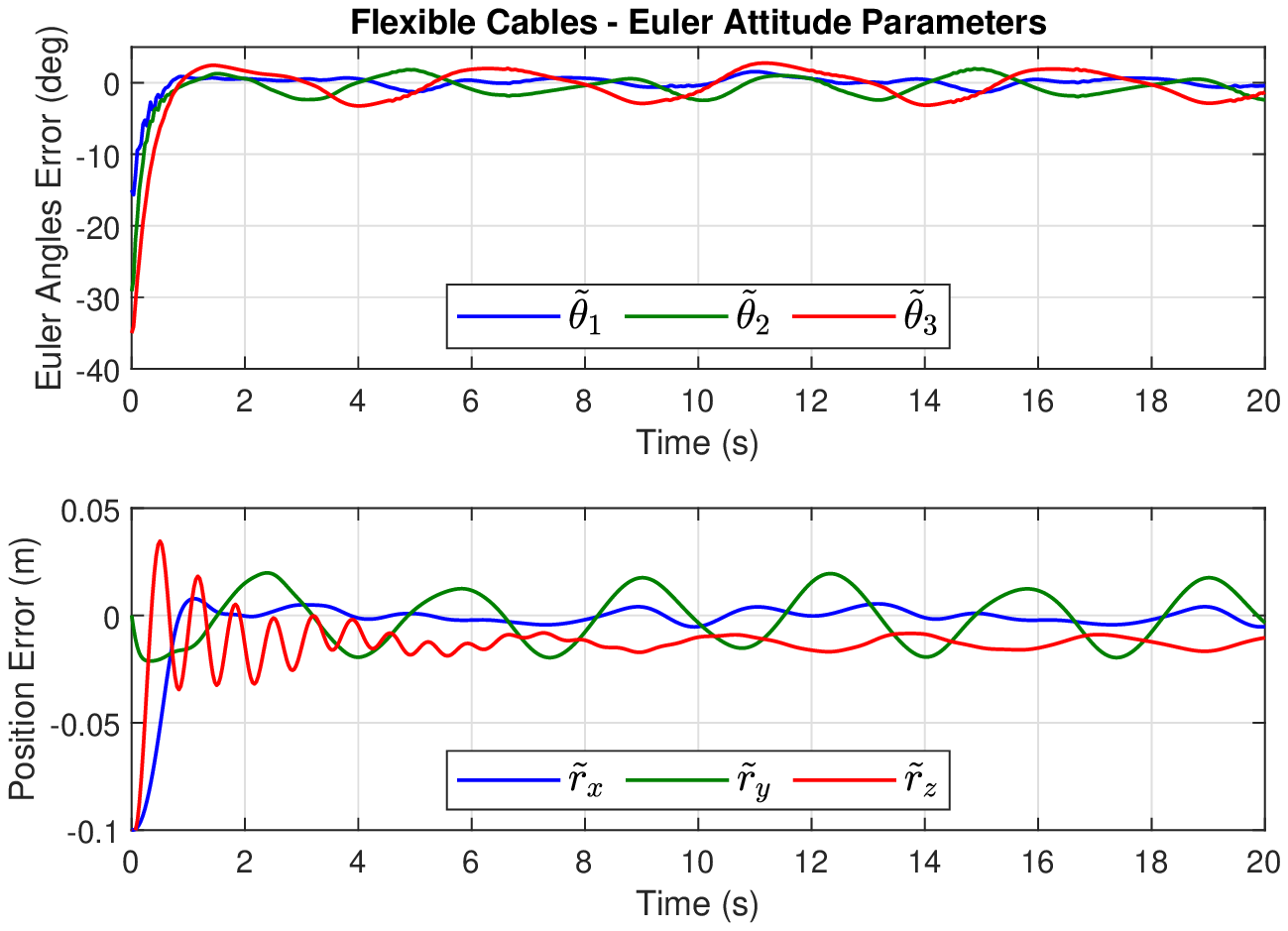}
        %  \caption{$y=3sinx$}
         \label{fig:Flex_GoodEuler_angx}
     }
     \subfigure[]{
         \centering
        \includegraphics[width=0.31\textwidth]{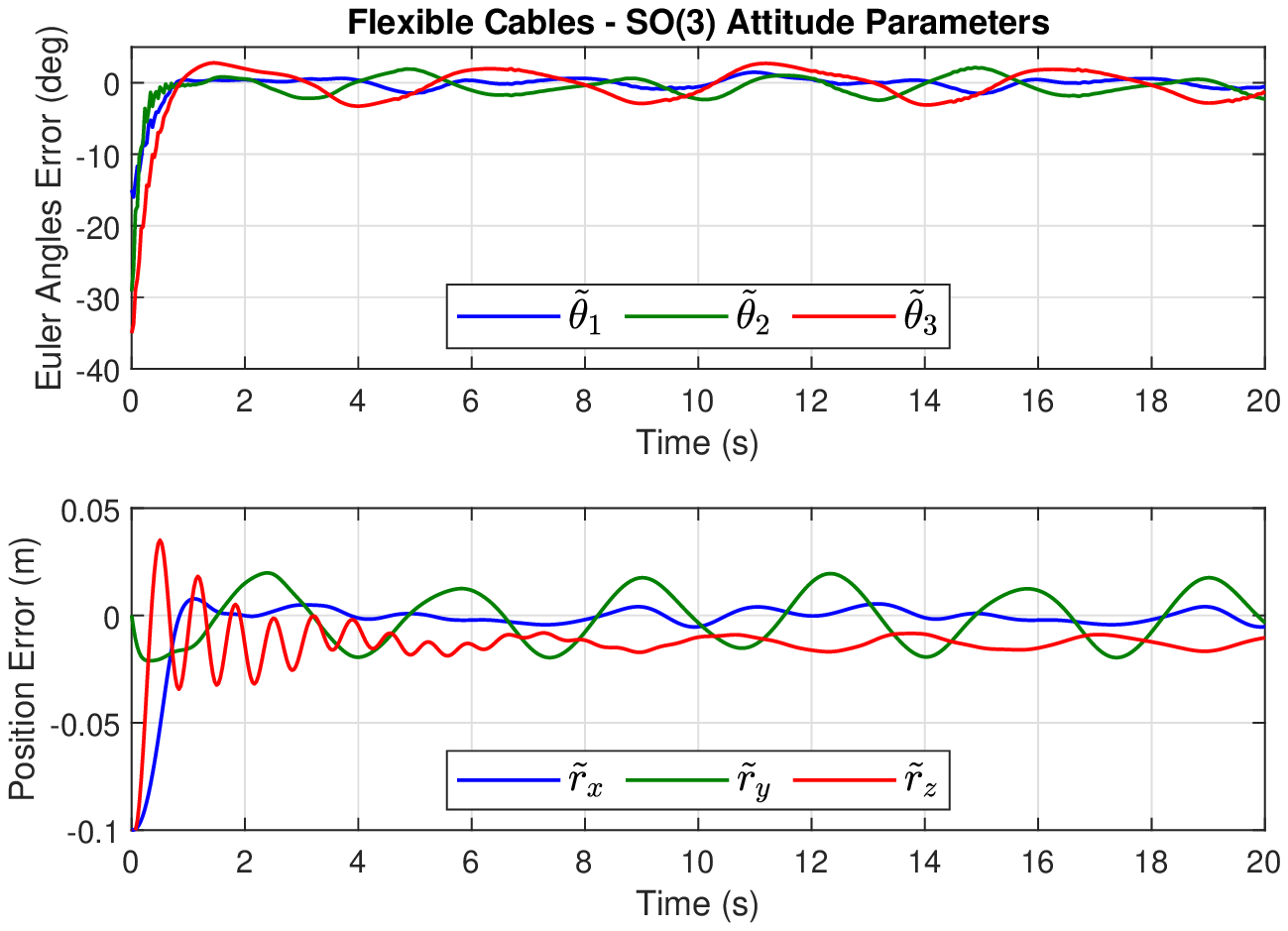}
        %  \caption{$y=5/x$}
         \label{fig:Flex_SO3_ang}
     }
    %  \subfigure[]{
    %      \centering
    %      \includegraphics[width=0.31\textwidth]{Figures2/Flex_badeuler_pos.eps}
    %     %  \caption{$y=x$}
    %      \label{fig:Flex_BadEuler_pos}
    %  }
    %  \subfigure[]{
    %      \centering
    %      \includegraphics[width=0.31\textwidth]{Figures2/Flex_euler_pos.eps}
    %     %  \caption{$y=3sinx$}
    %      \label{fig:Flex_GoodEuler_posx}
    %  }
    %  \subfigure[]{
    %      \centering
    %      \includegraphics[width=0.31\textwidth]{Figures2/Flex_SO3_pos.eps}
    %     %  \caption{$y=5/x$}
    %      \label{fig:Flex_SO3_pos}
    %  }
   %  \vspace{-8pt}
        \caption{Payload pose tracking errors in the flexible-cable simulations with (a) the simplified Euler-angle-based controller, (b) the correctly-implemented Euler-angle-based controller, and (c) the $SO(3)$-based controller. For visualization purposes, the attitude errors are plotted using a 3-2-1 Euler-angle sequence ($\tilde{\theta}_1$, $\tilde{\theta}_2$, and $\tilde{\theta}_3$). The position errors in the three axes of $\mathcal{F}_a$ are denoted as $\tilde{r}_x$, $\tilde{r}_y$, and $\tilde{r}_z$.}
        \label{fig:FlexCableResponse}
      %  \vspace{-10pt}
\end{figure*}

The control wrench, $\mbf{f}$, is distributed to the winch torques, $\mbs{\tau}$, using the improved closed-form solution from~\cite{Pott2018} as
\beq
\label{eq:tau_distribute}
\mbs{\tau} = \mbs{\tau}_{pt} +  \mbf{U}^\trans(\mbs{\theta}) \left(\mbf{f} - \mbs{\Pi}^\trans(\mbs{\rho})\mbs{\tau}_{pt}\right),
\eeq
where $\mbs{\tau}_{pt} = \text{diag}\{r_1,\ldots,r_8\} \mbf{t}_{pt}$ is a pretension torque, $r_i$ is the radius of the $i^\mathrm{th}$ winch, $\mbf{t}_{pt} \in \mathbb{R}^8$ contains the desired pretension in the $8$ cables, and $\mbf{U}(\mbs{\theta})$ %satisfies $\mbf{U}(\mbs{\theta}) \mbs{\Pi}(\mbs{\rho}) = \mbf{1}$ (i.e., $\mbf{U}(\mbs{\theta})$ 
is a pseudo-inverse of $\mbs{\Pi}(\mbs{\rho})$. A pretension value of $59$~N is used for each cable, with the goal of ensuring that the cable tensions are greater than $7.9$~N and less than $3937$~N. %This algorithm first computes the applied forces and moments on the payload $\mbf{f}_{m} = \mbs{\Pi}^\trans \mbf{f}_{pt}$ where $\mbf{f}_{pt}$ is the cable pretension. The winch torques are $\mbs{\tau} = (\mbf{f}_{pt} + \mbf{U}^\trans (\hat{\mbs{\tau}} - \mbf{f}_{m})) \cdot r_w$. 
If at a particular instance in time, a cable tension exceeds the allowed range, the algorithm sets the cable tension to the limiting value and recomputes~\eqref{eq:tau_distribute} with the row associated with that particular cable removed. %tension for other cables. 

%Robustness of the control system was further verified not only by having flexible cables but also by estimating the pose of the payload via the measurement of winch angles using recursive least square rigid cable fitting of the forward CDPR kinematics as performed in~\cite{nguyenAttitude}. 

Simulation results are presented in Figs.~\ref{fig:Path} through~\ref{fig:FlexCableResponse}, including detailed results for the case of flexible cables and the $SO(3)$-based controller in Figs.~\ref{fig:Path},~\ref{fig:Forces}, and~\ref{fig:FeedForwardPar}. Specifically, Fig.~\ref{fig:Path} features the desired payload pose and the closed-loop response of the payload pose, where $\mbf{r}^\trans(0) = \bbm 0 & 0 & 0.465 \ebm \mathrm{m}$ and the initial payload attitude is associated with a 3-2-1 Euler angle sequence with all angles equal to $-15$~deg. The CDPR's winch torques as a function of time are included in Fig.~\ref{fig:Forces} to demonstrate that positive cable tensions are maintained. Fig.~\ref{fig:FeedForwardPar} includes the estimated system parameters $\mbfhat{a}$ as a function of time.%, showing that the adaptation settles relatively quickly. 

The complete set of simulated controllers is compared by computing the axis-angle parameters associated with the attitude tracking error. The resulting error angle is plotted versus time in Fig.~\ref{fig:AxisAngleError}. To further quantify the differences in attitude tracking errors, the root mean square (RMS) value of the error angle is shown in Fig.~\ref{fig:rmsError} for the seven controllers and is separated into the RMS error of the transient response during the first 2 seconds of the simulation and the steady-state response after the first 2 seconds of the simulation. The comparisons in Fig.~\ref{fig:Errors} demonstrate that the Simplified Euler and Simplified FB Euler controllers lead to the least consistent tracking errors, particularly in their steady-state responses. %Fig.~\ref{fig:AxisAngleError} shows that the tracking errors with these controllers are smaller at certain times and larger at other times in comparison to the other attitude parameterizations. 
This is also evident in Fig.~\ref{fig:rmsError}, where the RMS attitude tracking errors are largest for these controllers after the first 2 seconds. For a visual comparison, the pose tracking errors versus time are included for the Simplified Euler, Euler, and $SO(3)$-based controllers with rigid cables in Fig.~\ref{fig:RigidCableResponse} and flexible cables in Fig.~\ref{fig:FlexCableResponse}. Quick convergence of the tracking error is seen with rigid cables in Fig.~\ref{fig:RigidCableResponse} and reasonably small tracking error is present with the flexible cables in Fig.~\ref{fig:FlexCableResponse}, which demonstrates the robustness of the proposed controller. The Simplified Euler controller features larger oscillations in tracking errors compared to both the Euler-angle and $SO(3)$-based controllers.

\section{Conclusion}
\label{sec:Conclusion}

This paper presented an adaptive passivity-based CDPR pose tracking controller for various attitude parameterizations. The benefit of performing CDPR pose tracking with carefully defined attitude errors was demonstrated, where a linearized Euler-angle parameterization was shown to yield inferior tracking results. %at tracking the pose trajectory compared to other attitude paramerizations, such as the quaternion or $SO(3)$. 
Closed-loop asymptotic convergence of the pose tracking error was proven and shown to be robust to parameter uncertainty through nonlinear stability analysis and also in simulation with a CDPR that featured unmodeled and uncertain flexible cable dynamics. %Numerical simulation of CDPR with flexible cables and rigid body pose estimate demonstrates the robustness of the controller. 

Future work will focus on experimental implementation of the proposed control law on multiple trajectories and explicit consideration of flexible cables in the controller formulation and stability analysis.


%%%%%%%%%%%%%%%%%%%%%%%%%%%%%%%%%%%%%%%%%%%%%%%%%%%%%%%%%%%%%%%%%%%%%%%%%%%%%%%%

%%%%%%%%%%%%%%%%%%%%%%%%%%%%%%%%%%%%%%%%%%%%%%%%%%%%%%%%%%%%%%%%%%%%%%%%%%%%%%%%

%%%%%%%%%%%%%%%%%%%%%%%%%%%%%%%%%%%%%%%%%%%%%%%%%%%%%%%%%%%%%%%%%%%%%%%%%%%%%%%%
% \section*{Appendix}
% \bdis
% \mbf{V} = \bbm \mbf{1}_{6\times6} & \mbf{0} & \mbf{0} & \mbf{0} & \mbf{0} & \mbf{0} \\ \bbm \bbm 0 & 1 & \mbf{0}_{1 \times 4} \ebm & \mbf{0} \\ \mbf{0} & \mbf{1}_{5 \times 5} \ebm \ebm
% \edis

% \section*{ACKNOWLEDGMENT}

% The authors would like to thank support funding from the University of Minnesota.

%%%%%%%%%%%%%%%%%%%%%%%%%%%%%%%%%%%%%%%%%%%%%%%%%%%%%%%%%%%%%%%%%%%%%%%%%%%%%%%%

% \begin{thebibliography}{99}
\bibliographystyle{IEEEtran}
\bibliography{IEEEabrv}

\end{document}